\documentclass{article}

\usepackage{microtype}
\usepackage{subcaption}
\usepackage{graphicx}
\usepackage{booktabs} 

\usepackage{pgfplots}
\pgfplotsset{compat=1.18}
\usetikzlibrary{arrows.meta}

\usepackage{hyperref}

\usepackage[accepted]{icml2024}

\usepackage{amsmath}
\usepackage{amssymb}
\usepackage{mathtools}
\usepackage{amsthm}

\usepackage[capitalize,noabbrev]{cleveref}
\usepackage{ifthen}

\newcommand{\ee}{\mathbb{E}}

\newcommand{\nn}{\mathbb{N}}

\newcommand{\rr}{\mathbb{R}}

\newcommand{\zz}{\mathbb{Z}}
\renewcommand{\epsilon}{\varepsilon}
\newcommand{\mrin}{{\mathrm{in}}}
\newcommand{\mrout}{{\mathrm{out}}}
\newcommand{\mat}[2]{\mathrm{Mat}_{#1\times #2}}
\newcommand{\bool}[1][]{%
  \ifthenelse{\equal{#1}{}}%
    {\boldsymbol{b}}
    {\boldsymbol{b}^{(#1)}}
}
\newcommand{\Id}{\mathbf{Id}} 
\newcommand{\sign}{\text{sign}}
\newcommand{\N}{\mathcal{N}}

\newcommand{\boolNorm}[1]{{\ensuremath{||#1||_1}}}
\newcommand{\errNorm}[1]{{\ensuremath{||#1||_\infty}}}
\newcommand{\feat}{\vec{\phi}}
\newcommand{\feature}{\feat}
\newcommand{\Features}{\Phi}
\newcommand{\readoff}{\vec{r}}
\renewcommand{\read}{\readoff}
\newcommand{\Readoffs}{\mathbf{R}}

\newcommand{\circuit}{\mathcal{C}}
\newcommand{\C}{\circuit}
\newcommand{\interf}{\mu}
\newcommand{\B}{\mathcal{B}}
\newcommand{\model}{{\mathcal{M}_{w}}}
\newcommand{\mlp}[1][]{%
  \ifthenelse{\equal{#1}{}}%
    {\mathrm{MLP}}
    {\mathrm{MLP}^{(#1)}}
}

\newcommand{\Wout}{{W_{\textrm{out}}}}
\newcommand{\Win}{{W_{\textrm{in}}}}
\newcommand{\Wbias}{{w_{\textrm{bias}}}}
\newcommand{\relu}{\mathrm{ReLU}}
\newcommand{\tO}{\tilde{O}}
\newcommand{\tOmega}{\tilde{\Omega}}
\newcommand{\tTheta}{\tilde{\Theta}}
\newcommand{\act}[1][]{%
  \ifthenelse{\equal{#1}{}}%
    {\vec{a}}
    {\vec{a}^{(#1)}}
}

\newcommand{\sparsity}{s}

\newcommand{\idx}{\Gamma}

\theoremstyle{plain}
\newtheorem{theorem}{Theorem}

\newtheorem{lemma}[theorem]{Lemma}
\newtheorem{corollary}[theorem]{Corollary}
\newtheorem{definition}{Definition}

\theoremstyle{remark}
\newtheorem{remark}[theorem]{Remark}
\newtheorem{example}{Example}

\usepackage[textsize=tiny]{todonotes}

\icmltitlerunning{Computation in Superposition}

\begin{document}

\twocolumn[
\icmltitle{Mathematical Models of Computation in Superposition}



\icmlsetsymbol{equal}{*}

\begin{icmlauthorlist}
\icmlauthor{Kaarel H\"{a}nni}{equal,cadenza}
\icmlauthor{Jake Mendel}{equal,apollo}
\icmlauthor{Dmitry Vaintrob}{equal}
\icmlauthor{Lawrence Chan}{}
\end{icmlauthorlist}

\icmlaffiliation{cadenza}{Cadenza Labs / Caltech.}
\icmlaffiliation{apollo}{Apollo Research}

\icmlcorrespondingauthor{Lawrence Chan}{chanlaw@berkeley.edu}
\icmlkeywords{Machine Learning, ICML, superposition, random projection, sparse boolean circuits}

\vskip 0.3in
]



\printAffiliationsAndNotice{\icmlEqualContribution} 

\begin{abstract}
Superposition -- when a neural network represents more ``features'' than it has dimensions -- seems to pose a serious challenge to mechanistically interpreting current AI systems. Existing theory work studies \emph{representational} superposition, where superposition is only used when passing information through bottlenecks. In this work, we present mathematical models of \emph{computation} in superposition, where superposition is actively helpful for efficiently accomplishing the task. 

We first construct a task of efficiently emulating a circuit that takes the AND of the $\binom{m}{2}$ pairs of each of $m$ features. We construct a 1-layer MLP that uses superposition to perform this task up to $\varepsilon$-error, where the network only requires $\tilde{O}(m^{\frac{2}{3}})$ neurons, even when the input features are \emph{themselves in superposition}. We generalize this construction to arbitrary sparse boolean circuits of low depth, and then construct ``error correction'' layers that allow deep fully-connected networks of width $d$ to emulate circuits of width $\tilde{O}(d^{1.5})$ and \emph{any} polynomial depth. 
We conclude by providing some potential applications of our work for interpreting neural networks that implement computation in superposition. 
\end{abstract}

\section{Introduction}\label{sec:intro}

Mechanistic interpretability seeks to decipher the algorithms utilized by neural networks \citep{olah2017feature, elhage2021mathematical, räuker2023transparent, olah2020zoom, meng2023locating, geiger2021causal, wang2022interpretability, conmy2024towards}. A significant obstacle is that neurons are \textit{polysemantic} -- activating in response to various unrelated inputs \citep{FUSI201666, nguyen2016multifaceted, olah2017feature, geva2021, goh2021multimodal}. As a proposed explanation for polysemanticity, \citet{olah2020zoom} introduce the `superposition hypothesis' (see also \citet{arora2018linear, elhage2022toy}): the idea that networks represent many more concepts in their activation spaces than they have neurons by sparsely encoding features as nearly orthogonal directions.

\begin{figure*}
    \centering
    \includegraphics[width=0.9\textwidth]{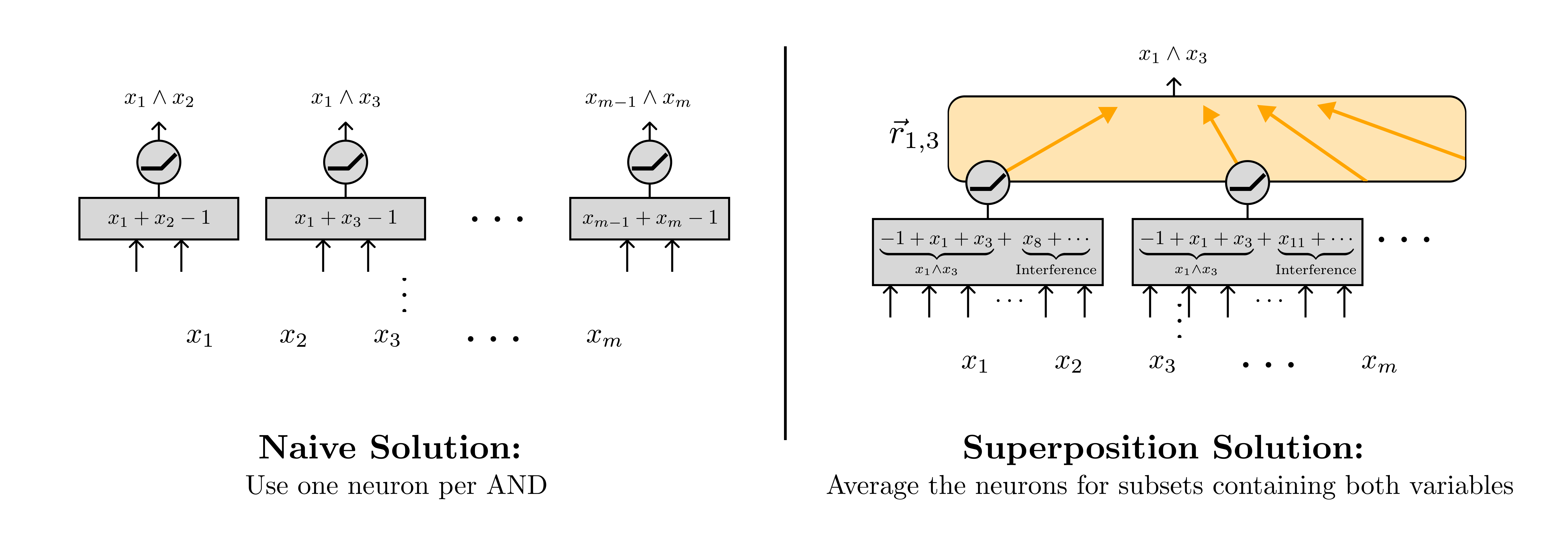}
    \vspace{-10pt}
    \caption{The naive way to linearly represent the pairwise ANDs of $m$ boolean variables using an MLP is to use one neuron to compute the AND of each pair of variables (left). This requires $\binom{m}{2} = O(m^2)$ neurons. However, when inputs are sparse, there is a much more efficient implementation using superposition (right). Here, each neuron checks for whether or not \emph{at least two variables} are active in a subset of random variables. Then, for any pair of variables, we can read off the AND of that pair by averaging together the activations of all neurons corresponding to the subsets containing both variables. With appropriately chosen subsets, we can $\varepsilon$-linearly represent all pairwise ANDs using only $\tO(m^{\frac{2}{3}})$ neurons, even when the inputs are themselves represented in superposition (Section~\ref{sec:one-layer-mlp-u-and}).}
    \label{fig:u-and-superposition}
\end{figure*}
Previous work has studied how networks can store more features than they have neurons in a range of toy models \cite{elhage2022toy, scherlis2022polysemanticity}. However, previous models of superposition either involve almost no computation \citep{elhage2022toy} or rely on some part of the computation not happening in superposition \citep{scherlis2022polysemanticity}. Insofar as neural networks are incentivized to learn as many circuits as possible \citep{olah2020zoom}, they are likely to compute circuits in the most compressed way possible. Therefore, understanding how networks can undergo more general computation in a fully superpositional way is valuable for understanding the algorithms they learn.

In this paper, we lay the groundwork for understanding computation in superposition \emph{in general}, by studying how neural networks can emulate sparse boolean circuits. 
\begin{itemize}
    \item In Section \ref{sec:formalism}, we clarify existing definitions of linearly represented features, and propose our own definition which is more suited for reasoning about computation. 
    \item In Section \ref{sec:u-and}, we focus our study on the task of emulating the particular boolean circuit we call the \textit{Universal AND} (U-AND) circuit. In this task, a neural network must take in a set of boolean features in superposition, and compute the pairwise logical ANDs of these features in a single layer with as few hidden neurons as possible. We present a construction which allows for many more new features to be computed than the number of hidden neurons, with outputs represented natively in superposition. We argue that real neural networks may well implement our construction in the wild by proving that randomly initialised networks are very likely to emulate U-AND. 
    \item In Section \ref{sec:sparse-boolean-circuit-model} we demonstrate a second reason why this task is worth studying: it is possible to modify our construction to allow a wide range of large boolean circuits to be emulated entirely in superposition, provided that they satisfy a certain sparsity property. 
\end{itemize}
We conclude with a discussion of the limitations of our formal models, including the fact that our results are asymptotic and deal with only boolean features, and provide directions of future work that could address them.

\section{Background and setup}\label{sec:formalism}
\subsection{Notation and conventions}

\paragraph{Asymptotic complexity and $\tO$ notation} We make extensive use of standard Bachmann–Landau (``big O'') asymptotic notation. We use $\tO$ to indicate that we are ignoring polylogarithmic factors:
\begin{align*}
    \tO(g(n)) := O(g(n) \log^k n) \textrm{\quad for some $k \in \zz$.}
\end{align*}
(And so forth for $\tTheta, \tOmega$, etc.)

\paragraph{Fully connected neural networks} We use $\model: X \rightarrow Y$ to denote a neural network model parameterized by $w$ that takes input $x \in X$ and outputs $\model(x) \in Y$. In this work, we study fully-connected networks consisting of $L$ MLP layers with ReLU activations:
\begin{align*}
    \act[0](x) &= x\\
    \act[l](x) &= \mlp[l](\act[l-1](x)) \\
    & = \relu(\Win^{(l)}\act[l-1](x) + \Wbias^{(l)})\\
    \model(x) &= \Wout \act[L],
\end{align*}
where $\relu(x) = \max(0, x)$ with max taken elementwise. We assume that our MLPs have width $d$ for all hidden layers, that is, $\act^{(l)} \in \rr^d$ for all $l \in \{1, ... ,L\}$. For simplicity's sake we will be dropping $l$ whenever we only talk about a single layer at a time. 
\begin{figure}
    \centering
    \vspace{10pt}
    \includegraphics[width=0.98\columnwidth]{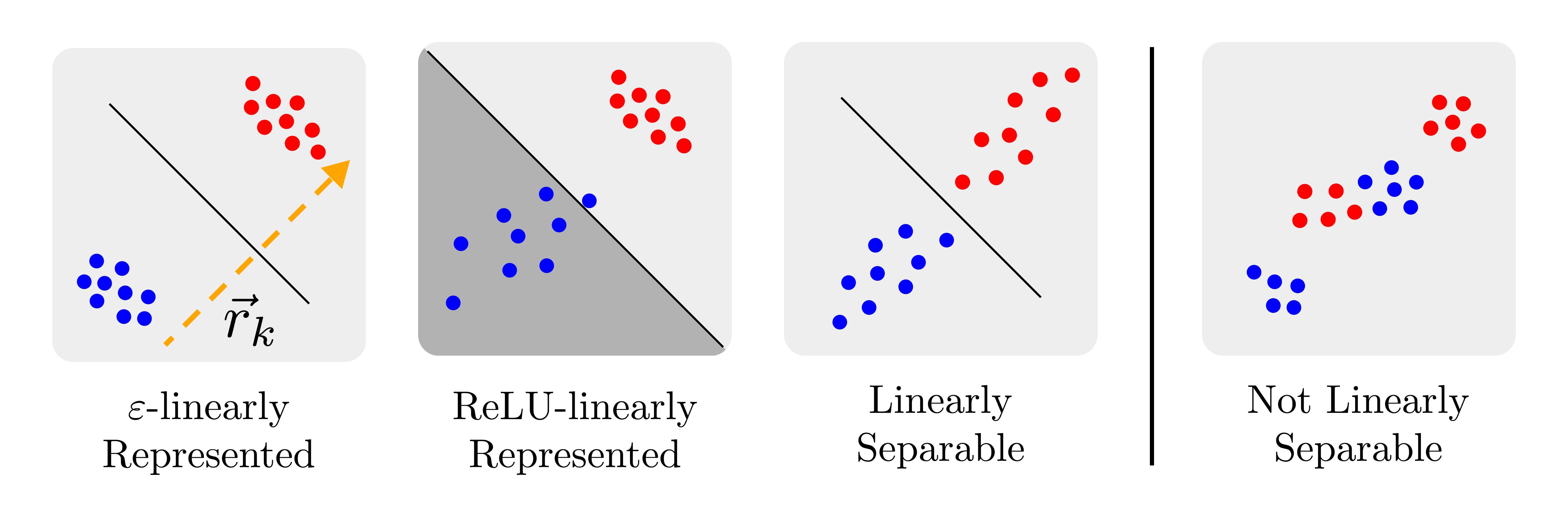}
    \vspace{-10pt}
    \caption{In Section~\ref{sec:linear-representations}, we distinguish between boolean features that are \emph{$\varepsilon$-linearly represented} (left), \emph{$\relu$-linearly represented} (center left), and those that are only \emph{linearly separable} (i.e. weakly linearly represented) (center right). Red/blue indicates the presence or absence of the feature. In addition to being linearly separable, $\varepsilon$-linearly represented features must satisfy the further condition that the variance in the readoff direction $\readoff_k$  \emph{within} the positive and negative clusters is small compared to the margin between the two.}
    \label{fig:eps-separated}\label{fig:linear-represented-definitions}
\end{figure}

\paragraph{Features and feature vectors} \label{sec:feature-def} Following previous work in mechanistic interpretability (e.g. \citet{tamkin2023codebook, rajamanoharan2024improving}), we suppose that the activations of a model can be thought of as representing $m > d$ boolean \emph{features} $f_k\colon X \rightarrow \{0, 1\}$ of the input \emph{in superposition}. That is, \begin{align*}
    \act(x) = \sum_{i = 1}^{m} \feat_k f_k (x) 
\end{align*} 
for some set of \emph{feature vectors} $\feat_1, ..., \feat_{m} \in \rr^d$ and features $f_1, ..., f_{m}\colon X \rightarrow \{0, 1\}$. Equivalently,
\begin{align*}
    \act^{(l)}(x) = \Features \bool
\end{align*}
where $\Features = (\feat_1, ..., \feat_{m})$ is the $d \times m$ \emph{feature encoding} matrix with columns equal to the feature vectors and $\bool \in \{0,1\}^m$ is the boolean vector with entries $\bool_k$ = $f_k(x)$. 

In addition, as in previous work, we assume that these features are \emph{$s$-sparse}, in that only at most $s \ll d, m$ features $f_i$ can be nonzero for any input $x$ (equivalently, $||\bool||_1 \leq s$.) For clarity, we preferentially use $k, \ell \in \{1, ..., m\}$ to index \emph{features} (in $\{0, 1\}^m)$ and $i, j \in \{1, ..., d\}$ to index the standard neuron basis of activations (in $\rr^d$).



\paragraph{Sparse boolean circuits}\label{sec:bool-circuits} \label{sec:boolean-circuit-setup}

We construct tasks where a neural network needs to emulate a boolean circuit $\circuit\colon \{0,1\}^m\to \{0,1\}^{m'}$. We assume that this circuit can be written as $\circuit = \circuit_L \circ \cdots \circ \circuit_1$, where each intermediate ``layer" $\C_l: \{0,1\}^m\to \{0,1\}^m$ is a collection of $m$ parallel boolean gates (of fan-in up to 2), for $l < L$. We say that a circuit $\circuit$ is $s$-sparse on boolean input $\bool \in \{0, 1\}^m$ if the input $\bool^{(0)} = \bool$ and all intermediate activations $\bool^{(l)} = \C_{i}(\bool^{(l-1)})$ are $s$-sparse, i.e. they satisfy $||\bool^{(i)}||_1 \leq s$.


\subsection{Strong and weak linear representations}\label{sec:linear-representations}
Given the activations of a neural network at a particular layer $a^{(l)}\colon X \rightarrow \rr^d$, we can also ask what features are \emph{linearly represented} by $a^{(l)}$. In this section, we present three definitions for a feature being linearly represented by $a^{(l)}$, which we illustrate in Figure~\ref{fig:linear-represented-definitions}.

The standard definition of linear representation is based on whether or not the representations of positive and negative examples can be separated by a hyperplane:
\begin{definition}[Weak linear representations]\label{def:weak_rep}
    We say that a binary feature $f_k$ is \emph{weakly linearly represented} by $a \colon X \rightarrow \rr^d$ (or \emph{linearly separable} in a) if there exists some $\readoff_k \in \rr^d$ such that for all $x_1, x_2 \in X$ where $f_k(x_1) = 0$ and $f_k(x_2) = 1$, we have:
    \begin{align*}
        \readoff_k\cdot a(x_1) < \readoff_k \cdot a (x_2).
    \end{align*}
    Or, equivalently, the sets $\{x | f_k(x) = 0\}$ and $\{x | f_k(x) = 1\}$ are separated by a hyperplane normal to $\readoff_k$. 
\end{definition}

That being said, features being linearly separable does not mean a neural network can easily ``make use" of the features. For some weakly linearly represented features $f_1$ and $f_2$, neither $f_1 \land f_2$ nor $f_2 \lor f_2$ need to be linearly represented, even if their read-off vectors $\readoff_1, \readoff_2$ are \emph{orthogonal}
(Figure~\ref{fig:why-eps-linear}). In fact, a stronger statement is true: it might not even be possible to linearly separate $f_1 \land f_2$ or $f_2 \lor f_2$ in $\mlp\circ a$, that is, even after applying an MLP to the activations (see Theorem~\ref{thm:mlp-composition-linearly-separable} in Appendix~\ref{app:more-feature-discussion}). 

As a result, in this paper we make use of a more restrictive notion of a feature being linearly represented:

\begin{definition}[$\varepsilon$-linear representations]\label{def:eps-linear} Let $X$ be a set of inputs and $a\colon X\to \mathbb{R}^d$ be the activations of a neural network (in a particular position/layer in a given model). We say that $f_1,\ldots,f_m$ are \emph{linearly represented with interference $\varepsilon$} (or \emph{$\varepsilon$-linearly represented} from these activation vectors) if there exists a \emph{read-off matrix} $\Readoffs \in \mat{m}{d}$ with rows $\readoff_1, \ldots,\readoff_m\in \mathbb{R}^d$ such that for all $k\in \{1, \ldots, m\}$ and all $x\in X$, we have 
\begin{align*}
    |\readoff_k \cdot \vec{a} (x) - f_k(x)| < \epsilon.
\end{align*}
We refer to $\readoff_k$ as a \emph{read-off vector} for the feature $f_k$. It follows that if $\act(x) = \sum_{i = 1}^{m} \feat_k f_k (x) $, then we have:
\begin{align*}
    ||\Readoffs \Features - \Id_m ||_\infty < \epsilon
\end{align*}
where $\Id_m$ is the $m\times m$ identity matrix\footnote{
In some cases if the feature vectors satisfy $|\Features^T \Features - \Id_m| \leq \mu$ — that is, if the feature vectors are \emph{almost orthogonal with interference $\mu$}, then the features vectors can function as their own readoffs.
}.
\end{definition}
For brevity's sake, we very slightly abuse notation here to include the bias term in $\readoff_k$. This is equivalent to assuming that one of $\act$'s outputs is a constant, that is, $a_i(x)=c$ for all x for some $i \in \{1, ..., d\}$ and some $c \in \rr$. 

In contrast to features that are merely linearly separable, features that are $\varepsilon$-linearly represented are easy to linearly separate, as we show in Figure~\ref{fig:why-eps-linear}. We formalize and prove this in Theorem~\ref{thm:mlp-composition-eps-linear} in Appendix~\ref{app:more-feature-discussion}.


\begin{figure}
    \centering
    \includegraphics[width=0.9\columnwidth]{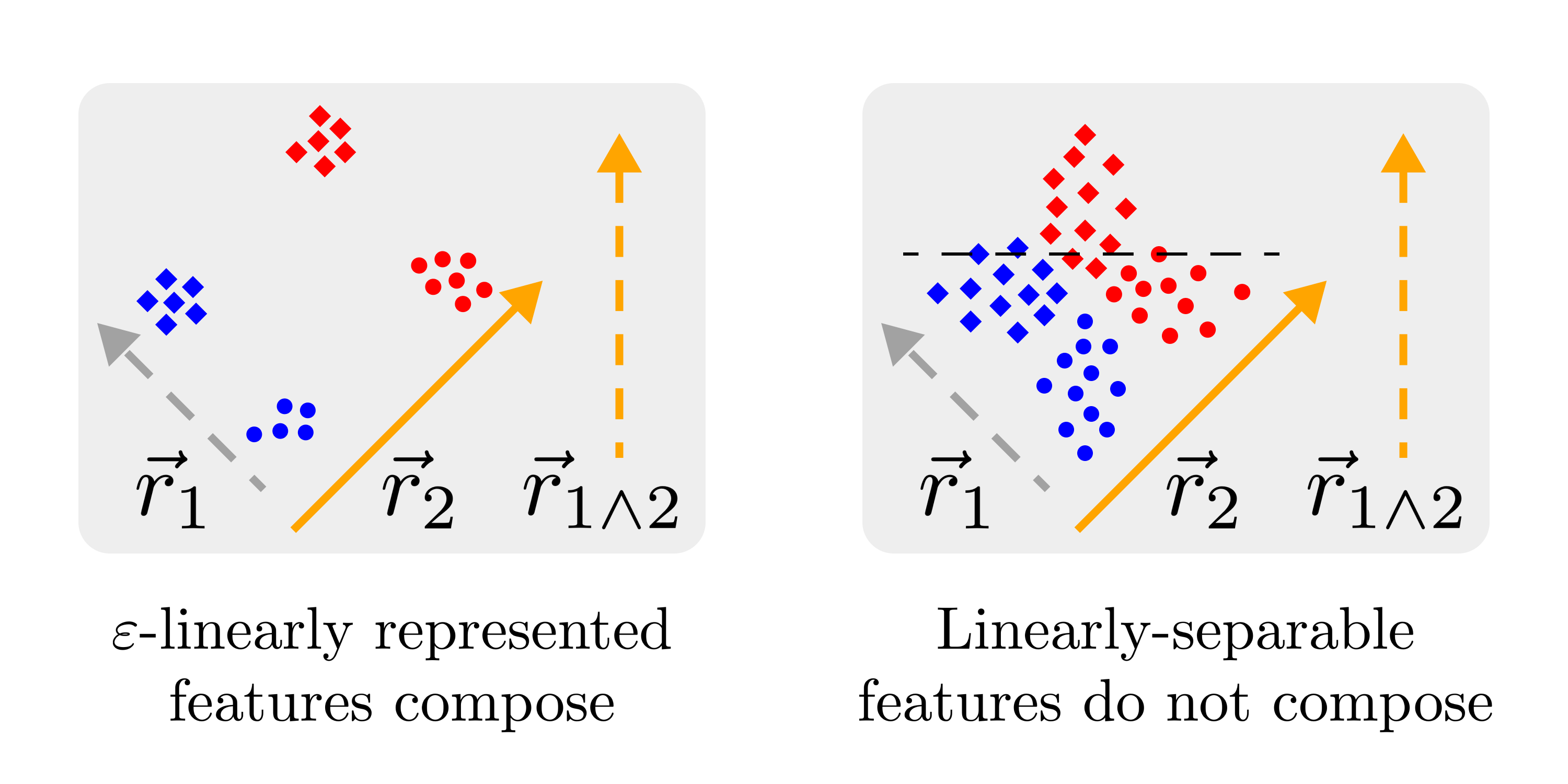}
    \vspace{-10pt}
    \caption{When two features $f_1, f_2$ are $\varepsilon$-linearly represented in activations $a(x)$, we can use two MLP neurons with input weights $\readoff_1, \readoff_2$ to read-off the two features, after which $f_1 \land f_2$ and $f_1 \lor f_2$ are $\epsilon$-linearly represented in the MLP activations $\mlp(a(x))$. However, because linearly-separable features can have arbitrarily small margin, there might exist \emph{no} MLP such that $f_1 \land f_2$ and $f_1 \lor f_2$ are linearly separable in $\mlp(a(x))$.}
    \label{fig:why-eps-linear}
\end{figure}


\paragraph{Comparison with Anthropic's Toy Model of Superposition} \label{sec:strong-vs-weak-linear-rep:comparison-anthropic} Finally, \citet{elhage2022toy} and \citet{bricken2023monosemanticity} consider a definition of linearly represented feature that involves using a ReLU to remove negative interference:
\begin{definition}[ReLU-linear representations]
    A set of $m$ binary features $\vec{F} = (f_1, ..., f_m)$ is ReLU-linearly represented in $a\colon X \rightarrow \rr^d$ with error $\varepsilon$ if there exists a read-off matrix $\Readoffs \in \mat{m}{d}$ such that 
    \begin{align*}
        \mathbb{E}_{x\in X} ||\vec{F}(x) - \relu\left (\Readoffs a(x) \right) ||_2 &< \varepsilon. 
    \end{align*}
\end{definition}
Note that in contrast to $\epsilon$-linearly represented features, where each individual feature must be able to be read off using an affine function with small error on \emph{every} datapoint, ReLU-linear representated features are read off using a MLP layer with $m$ neurons (one per feature), such that the \emph{expected} $\ell_2$ loss (summed across all $m$ features) is small. 


\label{sec:strong-vs-weak-linear-rep}




\section{Universal ANDs: a model of single-layer MLP superposition}\label{sec:u-and}\label{sec:one-layer-mlp-u-and}

We start by presenting one of the simplest non-trivial boolean circuits: namely, the one-layer circuit that computes the pairwise AND of the input features. Note that due to space limitations, we include only proof sketches in the main body and may ignore some regularity conditions in the theorem statement. See Appendix~\ref{app:proofs} for more rigorous theorem statements and proofs.

\begin{definition}[The universal AND boolean circuit]
Let $\bool \in \{0, 1\}^m$ be a boolean vector. The universal AND (or U-AND) circuit has $m$ inputs and $\binom{m}{2}$ outputs indexed by unordered pairs $k,\ell$ of locations and is defined by
$$\circuit_{\mathrm{UAND}}(\bool)_{k,\ell} : = \bool_k\land \bool_\ell.$$
In other words, we apply the AND gate to all possible pairs of distinct inputs to produce $\binom{m}{2}$ outputs. 
\end{definition}

We will build our theory of computation starting from a single-layer neural net that emulates the universal AND when the input $\bool$ is $\sparsity$-sparse for some $\sparsity \in \nn$ (this implies that the output has sparsity $O(\sparsity^2)$). 

\subsection{Superposition in MLP activations enables more efficient U-AND}

First, consider the naive implementation, where we use one ReLU to implement each AND using the fact that for boolean $x_1, x_2$:
\begin{align*}
    \relu(x_1 + x_2 -1) = x_1 \land x_2.
\end{align*}
This requires $\binom{n}{2} = O(n^2)$ neurons, each of which is monosemantic in that it represents a single natural feature. In contrast, by using sparsity, we can construct using exponentially fewer neurons (Figure~\ref{fig:u-and-superposition}):


\newcommand{\bone}{{\bool}_{k_1}{}}
\newcommand{\btwo}{{\bool}_{k_2}{}}
\newcommand{\bthree}{{\bool}_{k_3}{}}

\begin{theorem}[U-AND with basis-aligned inputs]\label{thm:uand-no-superposition-body}
Fix a sparsity parameter $s\in \nn.$ 
Then for large input length $m$, there exists a single-layer neural network $\model(x) = \mlp(x) = \relu(\Win x + \Wbias)$ that $\varepsilon$-linearly represents the universal AND circuit $\circuit_{\mathrm{UAND}}$ on $\sparsity$-sparse inputs, with width $d = \tO_m(1/\epsilon^2)$ (i.e. polylogarithmic in $m$).  
\end{theorem}

\begin{proof}(sketch) 
To show this, we construct an MLP such that each neuron checks whether or not at least two inputs in a small random subset of the boolean input $\bool$ are active (see also Figure~\ref{fig:u-and-superposition}). Intuitively, since the inputs are sparse, each neuron can be thought of as checking the ANDs of any pair of input variables $\bone, \btwo$ in the subset, with interference terms corresponding to all the other variables. That is, we can write the preactivation of each neuron as the sum of the AND of $\bone, \btwo$ and some interference terms:
\begin{align*}
    \underbrace{-1 + \bone + \btwo \vrule width0pt height0pt depth2.91ex}_{\bone \land \btwo} + \underbrace{\sum_{k' \not = k_1, k_2} \bool_{k'}}_{\textrm{interference terms}}
\end{align*}
We then use the sparsity of inputs to bound the size of the interference terms, and show that we can ``read-off" the AND of $\bone, \btwo$ by averaging together the value of post-ReLU activations of the neurons connected to $\bone, \btwo$. We then argue that this averaging reduces the size of the interference terms to below $\varepsilon$. 

Specifically, we construct input weights $\Win \in \mat{d}{m}$ such that the input to each neuron is connected to the $k$th entry of the input $\bool_k$ with weight 1 with probability $p = {\log^2{m}}/{\sqrt{d}}$, and weight 0 otherwise. We set the bias of each neuron to $-1$. 

Let $\idx(k)$ be indices of neurons that have input weight $1$ for $\bool_k$, and $\idx(k_1, k_2)$ be the indices of neurons that have input weight $1$ for $\bone, \btwo$, $\idx(k_1, k_2, k_3)$ be the indices of neurons reading from all of $\bone, \btwo, \bthree$, and so forth. By construction, $\idx(k_1)$ has expected size $\Theta(\log^2 m \cdot \sqrt{d})$, $\idx(k_1, k_2)$ has expected size $\Theta(\log^4 m)$, and $\idx(k_1, k_2, k_3)$ has expected size $\Theta(\log^6 m/\sqrt{d})$. In general, the set of indices for $n$ such inputs has expected size $\Theta(\log^{2n}/d^{(n/2-1)})$

Our read-off vector $\readoff$ for the AND $\bone \land \btwo$ will have entries:
\begin{align*}
    \readoff_{(i)} = \begin{cases}
        \frac{1}{|\idx(k_1, k_2)|} & i \in |\idx(k_1, k_2)|\\
        0 & \textrm{otherwise}
    \end{cases}
\end{align*}

We then check that $\readoff \cdot \mlp(\bool)$ gives the correct output in each of three cases. Note that for any input, $\readoff \cdot \mlp(\bool) \geq \bone \land \btwo$, so it suffices to upper bound the average number of non-$k_1, k_2$ inputs that are non-zero, divided by the total number of neurons in $\idx(k_1, k_2)$. 

\begin{itemize}

    \item When $\bone = \btwo=0$, the interference terms in each read-off neuron have value at most $s$, and there are at most 
    \begin{align*}
        \sum_{\bool_{k'} = \bool_{k''} = 1} |\idx(k_1, k_2, k', k'')| &= \Theta(s^2 \cdot \log^8 m / d) 
    \end{align*}
    such neurons outputting non-zero values. So the error is bounded above by  
    \begin{align*}
        \frac{\sparsity \cdot \sum_{k' \not = k_1, k_2} |\idx(k_1, k_2, k', k'')|}{|\idx(k_1, k_2)|}
         &= \Theta(s^3 \cdot {\log^4 m}/{d}). 
    \end{align*}
    \item When $\bone = 1$ or $\btwo=1$, the interference terms in each read-off neuron have value at most $\sparsity-2$, and there are at most 
    \begin{align*}
        \sum_{\bool_{k'} = 1} |\idx(k_1, k_2, k')| &= \Theta(s \cdot \log^6 m / \sqrt{d}) 
    \end{align*}
    neurons that have such interference terms. 
    
    So the error is bounded above by
    \begin{align*}
        \frac{\sparsity-2}{|\idx(k_1, k_2)|}
        \sum_{k' \not = k_1, k_2} |\idx(k_1, k_2, k')| &\\
        = \Theta(s^2 \cdot & {\log^2 m}/{\sqrt{d}})
    \end{align*}

\end{itemize}
Combining the above, we get that the read-off error is $O(\log^4 m /\sqrt{d})$, and so setting $d = \Theta(\log^8 m/\varepsilon^2) = \tO_m(1/\varepsilon^2)$ gives us an error that is $<\varepsilon$ outside negligible probability. 

    
\end{proof}


\subsection{Neural networks can implement efficient U-AND even with inputs in superposition}

Note that in Theorem~\ref{thm:uand-no-superposition-body}, we assume that the network gets $m$ basis-aligned inputs (that is, not in superposition). However, it turns out that we can extend the result in Theorem~\ref{thm:uand-no-superposition-body} to inputs in superposition. 

\begin{theorem}[U-AND with inputs in superposition]\label{thm:uand-superposition-body}
Let $s \in \nn$ be a fixed sparsity limit and $\varepsilon <1$ a fixed interference parameter.
There exists a feature encoding $\Features$ and single-layer neural net $\model(x) = \mlp(x) = \relu(\Win x + \Wbias)$ with input size and width $d = \tO(\sqrt{m}/\varepsilon^2)$, where $\model\circ \Features$ $\epsilon$-linearly represents $\circuit_{\mathrm{UAND}}$ on all $s$-sparse inputs $\bool$.
\end{theorem}

\begin{proof} (sketch) 
    By picking almost orthogonal unit-norm vectors $\Features = (\feat_1,\dots, \feat_m)$, we can recover each feature up to error $\varepsilon$ using readoffs $\Readoffs = \Phi^T$. Take the input weight $\Win \in \mat{d}{m}$ for the MLP constructed in the proof of Theorem~\ref{thm:uand-no-superposition-body}. Using $\Win' = \Win \Readoffs$ and $\Wbias' = \Wbias$ suffices, as this gives us 
    \begin{align*}
        \model\circ \Features (\bool) &= \relu(\Win \Readoffs \Features \bool + \Wbias)\\
        &\approx \relu(\Win \bool + \Wbias),
    \end{align*}
    which is just the model from Theorem~\ref{thm:uand-no-superposition-body}, which $\varepsilon$-linearly represents $\circuit_{\mathrm{UAND}}$ as desired. Carefully tracking error terms shows that we need $d = \tTheta(\sqrt{m})$ neurons. 
\end{proof}



\label{thm:ts-and}



\subsection{Randomly initialized neural networks linearly represent U-AND}

While the results in previous section show that there exist \emph{some} network weights that $\varepsilon$-linearly represents the U-AND circuit $\circuit_{\mathrm{UAND}}$, there still is a question of whether neural networks can learn to represent many ANDs starting from the standard initialization. In this section, we provide some theoretical evidence -- namely, that sufficiently wide \emph{randomly initialized} one-layer MLPs $\varepsilon$-linearly represent $\circuit_{\mathrm{UAND}}$. 

\begin{theorem}[Randomly initialized MLPs linearly represent U-AND]\label{thm:random-uand}
    Let $\mlp: \rr^m \rightarrow \rr^d$ be a one-layer MLP with $d = \tOmega(1/\varepsilon^2)$ neurons that takes input $\bool$, and where $\Win$ is drawn i.i.d from a normal distribution $\mathcal N(0, \delta^2)$ and $\Wbias = \vec{0}$. Then this MLP $\varepsilon$-linearly represents $\circuit_{\mathrm{UAND}}$ on $\sparsity$-sparse inputs outside of negligible probability.  
\end{theorem}
\begin{proof}(Sketch)
    We prove this by constructing a read-off vector $\readoff$ for each pair of features $k_1, k_2$. Let $\sigma$ be the sign function 
    \begin{align*}
        \sigma(x) = \begin{cases}
            +1 & x > 0\\
            ~~0 & x = 0\\
            -1 & x < 0
        \end{cases}
    \end{align*}
    and let $w_{i,k}$ be the contribution to the preactivation of neuron $i$ from $\bool_k$.

    We construct $\readoff$ coordinatewise (that is, neuron-by-neuron). In particular, we set the $i$th coordinate of $\readoff$ to be
    \begin{align*}
        \readoff_i = \eta_i \left (\mathbf{1}_{\sigma(w_{i, k_1)} = \sigma(w_{i, k_2})} - \mathbf{1}_{\sigma(w_{i, k_1})\not = \sigma(w_{i, k_2})} \right ).
    \end{align*}
    That is, if $k_1$ and $k_2$ contribute to the neuron preactivations with the same sign, then  $\readoff_i = \eta_i$, else, $\readoff_i = -\eta_i$. Here, $\eta_i$ is a scaling parameter of size $\Theta({\sqrt{\sparsity}}/{d})$ used to scale the read-off to be $1$ when $\bone = \btwo = 1$

    When $\bone=0$ or $\btwo = 0$, the expected value of $\readoff \cdot \model(\bool)$ is zero, while the error terms have size $\tO_m(1/\sqrt{d})$. So setting $d = \tOmega({1/\varepsilon^2})$ suffices to get error below $\varepsilon$ with high probability. 

    When $\bone=\btwo=1$, the contribution from each neuron $i$ to $\readoff \cdot \model(\bool)$ with $\sigma(w_{i, k_1)} = \sigma(w_{i, k_2})$ will be in expectation larger than those with $\sigma(w_{i, k_1)} \not = \sigma(w_{i, k_2})$ (as the standard deviation of the sum of two weights with equal signs is larger than the sum of two weights with different signs, and we apply a ReLU). By setting $\eta$ to be the reciprocal of the difference in expected contributions, we have that this value has expectation 1. Again, as the error terms have size $\tO_m(1/\sqrt{d})$, it follows that setting $d = \tOmega({1/\varepsilon^2})$ suffices to get error below $\varepsilon$ with high probability, as desired. 
    
\end{proof}

Before proceeding, we record a corollary, which underscores the surprisingly strong asymptotic representability of the universal AND circuit.
\begin{corollary}
For any fixed input size $s,$ dimension $d$ and $m = d^{O(1)}$ polynomial in $d$, there exists a ``universal AND'' model with hidden dimension $d,$
$$\model: x\mapsto \relu(\Win(x))$$ from $\rr^d$ to $\rr^d$ and a feature matrix $\Features\in \mat{m}{d}$ such that for any input $\bool$ with sparsity $\boolNorm{\bool} = s,$ we have that $\model(\Features(\bool))\in \rr^d$ strongly linearly represents $\text{uAND}(\bool)\in \{0,1\}^{\binom{m}{2}}$ (with error at worst $\varepsilon = \tO\big(\frac{1}{\sqrt{d}}\big)$).
\end{corollary}





\newcommand{\cuand}{\circuit_{\mathrm{UAND}}^{(n)}}
\section{MLPs as representing sparse boolean circuits}\label{sec:sparse-boolean-circuit-model}
In the previous section we showed variants of computation in superposition at a single layer, for one of the simplest non-trivial boolean circuits. In this section, we extend these results to show that neural networks can efficiently represent \emph{arbitrary} sparse boolean circuits.

As in Section~\ref{sec:u-and}, we include only proof sketches in the main body due to space limitations, and may also ignore some regularity conditions in our theorem statements. See Appendix~\ref{app:proofs} for more rigorous theorem statements and proofs.

\subsection{Boolean circuits in single layer MLPs}\label{sec:sparse-circuit-single-mlp}
We start by extending these results from Section~\ref{sec:one-layer-mlp-u-and} to ANDs of more than two variables.

 Let $\circuit_{\mathrm{UAND}}^{(n)}$ be the boolean circuit of depth {$L = \log(n)$} that computes the ANDs of each $n$-tuple of elements in $\bool$.\footnote{Note that by our definition, boolean circuits are made of gates of fan-in at most 2. So computing the ANDs of $n$ variables requires a boolean circuit of depth $\log(n)$. }

\begin{lemma}[``High fan-in" U-AND]\label{lem:high-fan-in}
For each $n \in \nn$, there exists a one-layer neural network $\model = \mlp: \rr^m \rightarrow \rr^d$ with width $d = \tO(n /\varepsilon^2)$ such that $\model(\bool)$ $\varepsilon$-linearly represents $\circuit_{\mathrm{UAND}}^{(n)}$ on $\sparsity$-sparse inputs.
\end{lemma}
\begin{proof}(sketch)
    We can extend the construction in the proof of Theorem~\ref{thm:uand-no-superposition-body} to allow for ANDs of exactly $n$
    variables, by considering index sets $\mathbf{I}(k_1, k_2, ..., k_n)$ of n variables, and changing the bias of each neuron from $-1$ to $-n+1$. The expected size of an index set of $n$ variables is $\ee[|\mathbf I(k_1, k_2, ..., k_n)|] = p^n d$, and we require this expected value to be $\Omega(\log^4 m)$ to ensure that the index set is non-empty outside negligible probability (using the normal Chernoff and Union bounds). Therefore, we have to scale up the probability that any given value in $\Win$ is $1$: $p=\frac{\log^2 m}{d^{1/n}}$ suffices. A similar argument to the one found in the proof of Theorem~\ref{thm:uand-no-superposition-body} shows that all the interference terms are $o(1)$.
\end{proof}

As illustrated in Figure~\ref{fig:and_into_arbitrary}, Lemma~\ref{lem:high-fan-in} allows us to construct MLPs that $\varepsilon$-linearly represents arbitrary small circuits:

\begin{figure}
    \centering
    \includegraphics[width=0.8\columnwidth]{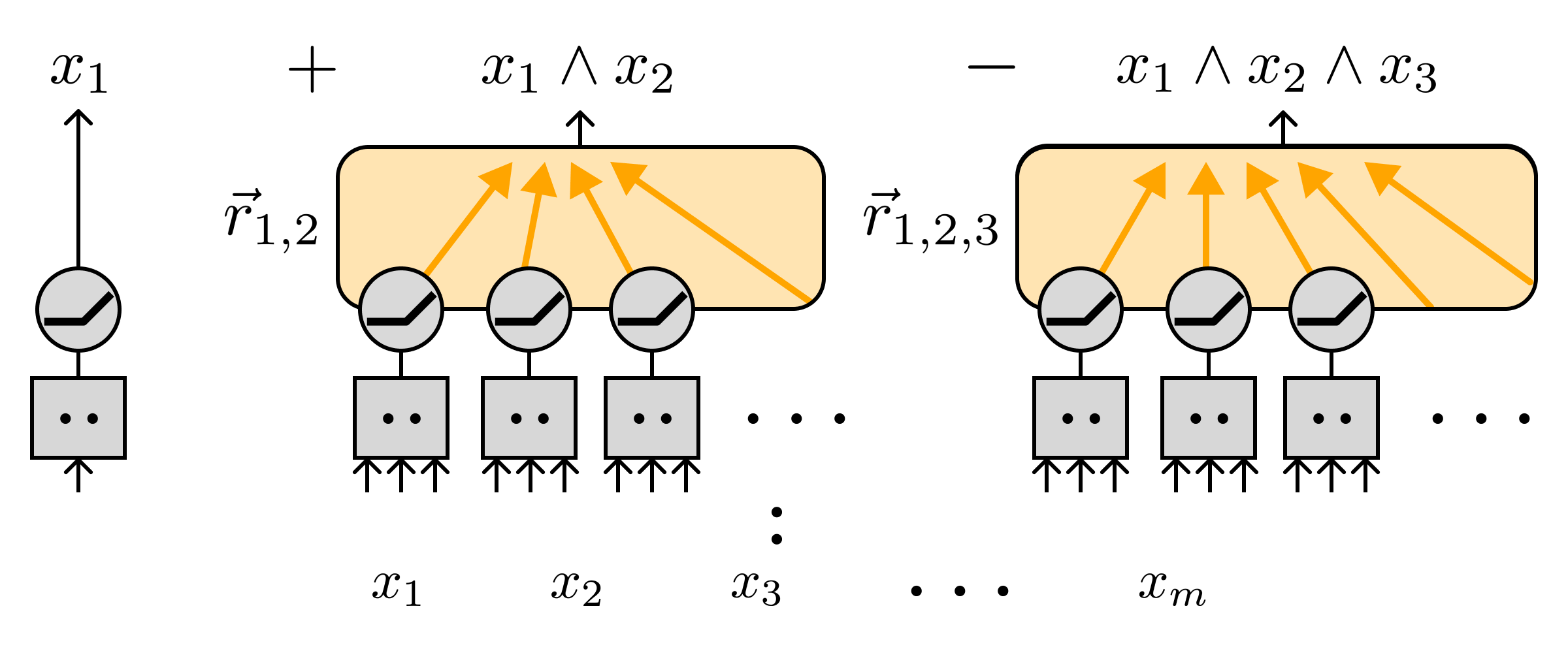}
    \caption{As discussed in Section~\ref{sec:sparse-circuit-single-mlp}, our U-AND construction can be extended to allow for arbitrarily high fan-in ANDs, which in turn allows for single-layer MLPs that linearly represent all small boolean circuits.}
    \label{fig:and_into_arbitrary}
\end{figure}
\begin{theorem}\label{thm:one-layer-sparse-circuit}
For any $s$-sparse circuit $\circuit$ of width $m$ and depth $L$, there exists a feature encoding $\Features \in \mat{d}{m}$ and a single-layer neural network $\model(x) = \relu(\Win x + \Wbias)$ of width $d = \tO(\sqrt{m})$ such that $\model(\Features \bool)$ $\varepsilon$-linearly represents $\circuit(\bool)_k$ for all $k \in \{1, ..., m\}$ for some $\varepsilon = \tO(m^{-1/3})$.
\end{theorem}
\begin{proof}(sketch)
    First, apply the construction in Theorem~\ref{thm:uand-superposition-body} to show that there exists one-layer MLPs of width $d = \tO(\sqrt{m})$ that compute $\circuit_{\mathrm{UAND}}^{(n)}$ when the inputs are in superposition, where $n \in \{2, 3, ..., 2^L\}$.

    Next, concatenate together the $2^L-1$ networks of width $d = \tO(\sqrt{m})$ that $\varepsilon$-linearly represent each 
    $\circuit_{\mathrm{UAND}}^{(n)}$ for $n \in \{2, 3, ..., 2^L\}$, when the inputs are in superposition. Since the output of \emph{any} boolean circuits of depth $L$ can be written as a linear combinations of ANDs of maximum fan-in $2^L$, it follows that the concatenated network $\varepsilon'$-linearly represents any boolean circuit of depth $L$, for some $\varepsilon'$ dependent on how many ANDs need to be added together to compute the circuit, as desired. 
\end{proof}

\subsection{Efficient boolean circuits via deep MLPs}
The one-layer MLP in Theorem~\ref{thm:one-layer-sparse-circuit} has width that is exponential in the depth of the circuit. However, by combining pairwise U-AND layers (which linearly represent any one-layer boolean circuit) with ``error correction" layers, we can construct deeper neural networks with sublinear width and depth linear in the depth of the circuit.

\begin{lemma}\label{lem:error-correction-layer}
Assume that $m=\tO(d^{1.5}),$ and $c$ is some large polylog constant.
    Then for sufficiently small input interference $\epsilon = \tO(1/\sqrt{d})$ there exists a 1-layer MLP $\model: \rr^d \rightarrow \rr^d$ that takes as input a boolean vector of length $m$ encoded in $d$-dimensions using superposition and returns (outside negligible probability) an encoding of the same boolean vector with interference $\epsilon/c$. 
\end{lemma}

\begin{proof} See Theorem~\ref{thm:err-correction-appendix} in Appendix~\ref{sec:error-correction}.

\end{proof}

By alternating between such ``error correction" layers and U-AND layers, we can construct more efficient circuits:

\begin{theorem}\label{thm:deep-mlp-circuit}
    Let $\circuit: \{0, 1\}^m\rightarrow \{0, 1\}^m$ be a circuit of width $m$ and of depth $L = O(m^c)$ polynomial in $m$.  There exists a neural network of width $d = \tO(m^{\frac{2}{3}} s^{2})$ and with depth $2L$ such that $\model(\Features \bool)$ $\varepsilon$-linearly $\circuit(\bool)$ for all but a negligible fraction of inputs $\bool$ on which $\circuit$ is $s$-sparse. 
\end{theorem}
\begin{proof}(sketch)
    As a single MLP layer can $\varepsilon$-linearly represent the ANDs of all input features (by Theorem~\ref{thm:uand-superposition-body}), we can use one MLP layer to approximate each layer of the circuit. However, the naive construction suffers from (potentially) exponentially growing error. To fix this, we insert an error correction layer from Lemma~\ref{lem:error-correction-layer} between every such layer. 
\end{proof}



\section{Related Work}\label{sec:related-work}
\label{sec:related-work:superposition}
The idea that neural networks could or should make use of distributed or compositional representations has been a mainstay of early neural network research \cite{rosenblatt1961principles, holyoak1987parallel, fodor1988connectionism}. 
\citet{arora2018linear} were the first in the modern deep learning context to discuss that neural networks could store many features in superposition. \citet{olah2020zoom} developed this idea into the `superposition hypothesis': the conjecture that networks use the same neurons for multiple circuits to maximise the number of circuits they can learn. 

Many of our results are similar in flavor to those from the fields of sparse dictionary \cite{tillmann2014computational} and hyperdimensional computing \cite{zou2021spiking}, as all rely on useful properties of high-dimensional spaces. In addition, many of our boolean circuit results on randomly-initialized MLP layers are similar in flavor to universality results on randomly initialized neural networks with different non-linearities \cite{rahimi2008uniform, rahimi2008weighted}. However, these results consider cases where there are fewer ``true features" than there are dimensions, while the superposition hypothesis requires that the number of ``true features" exceeds the dimensionality of the space. Randomized numerical linear algebra \cite{murray2023randomized} studies the use of random projections to perform efficient computation, but in the context of reducing the cost of linear algebra operations such as linear regression or SVD with inputs and outputs represented in an axis-aligned fashion. 

Superposition has been studied in a range of idealised settings: \citet{elhage2022toy} provided the first examples of toy models which employed superposition to achieve low loss and \citet{henighan2023superposition} further explored superposition in a toy memorisation task. Notably, they study features that are ReLU-linear represented. (See Section~\ref{sec:strong-vs-weak-linear-rep:comparison-anthropic} for more discussion.) \citet{scherlis2022polysemanticity} study a model of using a small number of neurons with \emph{quadratic} activations to approximately compute degree two polynomials. The models studied in all of these papers require sparse features of \textit{declining} importance. In contrast, our model allows for sparse features that are equally important. More importantly, none of these listed works study performing computation with \emph{inputs in superposition}. 

Several papers have also explored the prevalence of superposition in language models. \citet{gurnee2023finding} found that some bigrams are represented on sparse sets of neurons but not on any individual neurons. There is also a growing literature on using sparse dictionary learning to identify features in language models inspired by the superposition hypothesis \citep{cunningham2023sparse,  bricken2023monosemanticity, tamkin2023codebook, bloom2024gpt2residualsaes, braun2024identifying, templeton2024scaling} although it is unclear how much evidence the success of sparse dictionary learning in finding human-interpretable features provides for the superposition hypothesis.









\section{Discussion}\label{sec:conclusion}

\subsection{Summary}
In this work, we have presented a mathematical framework for understanding how neural networks can perform computation in superposition, where the number of features computed can greatly exceed the number of neurons. We have demonstrated this capability through the construction of a neural network that efficiently emulates the Universal AND circuit, computing all pairwise logical ANDs of input features using far fewer neurons than the number of output features. Furthermore, we have shown how this construction can be generalized to emulate a wide range of sparse, low-depth boolean circuits entirely in superposition. This work lays the foundation for a deeper understanding of how neural networks can efficiently represent and manipulate information, and highlights the importance of considering computation in superposition when interpreting the algorithms learned by these systems.



\subsection{Practical Takeaways for Mechanistic Interpretability}
Our primary motivation for undertaking this work was to glean insights about the computation implemented by neural networks. While we provide more potential takeaways in Appendix~\ref{appendix:pontify}, here we discuss what we think are two salient takeaways for interpretability: 

\paragraph{Unused features}
The implementation of U-AND by random matrices (Theorem \ref{thm:random-uand}) suggests that certain concepts may be detectable through linear probes in a network's activation space without being actively utilized in subsequent computations. This phenomenon could explain the findings of \citet{marks2024xor}, who observed that arbitrary XORs of concepts can be successfully probed in language models. Furthermore, it implies that successfully probing for a concept and identifying a direction that explains a high percentage of variance (e.g., 80\%) may not constitute strong evidence of the model's actual use of that concept. Consequently, there is reason to be cautious about how many of the features identified by Sparse Autoencoders \citep{cunningham2023sparse, bricken2023monosemanticity, bloom2024gpt2residualsaes, templeton2024scaling} are actively employed by the model in its computation.

\paragraph{Robustness to noise}
This research underscores the critical role of error correction in networks performing computations in superposition. Effective error correction mechanisms should enable networks to rectify minor perturbations in their activation states, resulting in a nonlinear response in output when activation vectors are slightly altered along specific directions. Expanding on this concept, \citet{heimersheim2023activation} conducted follow-up investigations, revealing the presence of \textit{plateaus} surrounding activation vectors in GPT2-small \citep{radford2019language}. Within these plateaus, model outputs exhibit minimal variation despite small changes in activation values, providing weak evidence for an error correcting mechanism in the model's computation.

\subsection{Limitations and future work}\label{sec:limitations}
That being said, there are a number of ways in which the computational framework presented in this work is very likely to miss the full richness of computation happening in any given real neural network. 

Firstly, this work studies computation on binary features. It is plausible that other kinds of features -- in particular, discrete features which take on more than $2$ distinct values, or continuous-valued features -- occur commonly in real neural networks. It would be valuable to extend the understanding developed in this work to such non-binary features. 

Secondly, though we do not require features to have declining importance, we do require features to be sparse, with each data point only having a small number of active features. It is plausible that not all features are sparse in practice (given the present state of empirical evidence, it even appears open to us whether a significant fraction of features are sparse in practice) -- for instance, perhaps real neural networks partly use more compositional representations with dense features. 

Thirdly, in this work, we have made a particular choice regarding what it takes for a feature to be provided in the input and to have been computed in the output: $\varepsilon$-linear representation (Definition~\ref{def:eps-linear}). Future empirical results or theoretical arguments could call for revising this choice --- for instance, perhaps an eventual full reverse-engineering picture would permit certain kinds of non-linear features. 

Finally and least specifically, the way of looking at neural net computation suggested in this work could turn out to be thoroughly confused. We consider there to be a lot of room for the development of a more principled and empirically grounded picture.



\section*{Impact Statement}
The primary impact of our work is to advance the field of mechanistic interpretability. While advancing this field may have many potential societal impacts, we feel that there are no direct, non-standard impacts of our work that are worth highlighting. 

\nocite{langley00}

\bibliography{main}
\bibliographystyle{icml2024}

\newpage
\appendix
\onecolumn

\section{Mathematical definitions}\label{sec:appendix-definitions}
Here, we list and define the mathematical terms that we use throughout this work. 
\begin{table}[h]
    \centering
\renewcommand{\arraystretch}{1.2}
    \begin{tabular}{p{0.3\linewidth} p{0.6\linewidth}}
        $X$ & set of inputs  \\
        $Y$ & set of outputs\\
        $\model: X \rightarrow Y$ & neural network with $\relu$ activations, parameterized by $w$\\
        $\act^{(l)}(x) \in \rr^d$ & the activations of a neural network at layer $l$, $l \in \{0, ..., L\}$\\
        $\mlp[l]: \rr^d \rightarrow \rr^d$ & the $l$th MLP layer, $\mlp[l](x) = \relu(\Win^{(l)}x + \Wbias^{(l)})$\\
        $f_k: X\rightarrow \{0,1\}$ & boolean feature of the input, $k = 1,\dots, m$\\
        $F: X \rightarrow \{0, 1\}^m$ & the concatenation of $m$ boolean features\\
        $\feat_k \in \rr^d$ & vector linearly representing the $k$th boolean feature \\\
        $\Features \in \rr^{d \times m}$ & the feature embedding matrix, $\Features = (\feature_1, ..., \feature_m)$\\
        $\bool = \bool(x)\in \{0,1\}^m$ & a boolean vector of length $m$ associated to an input/activation\\
        $\bool_k = \bool_k(x)\in \{0,1\}$ & the $k$th entry in the boolean vector, equal to $f_k(x)$\\
        $\boolNorm{\bool(x)}$ & ``sparsity'', a.k.a.\ number of bits that are ``on'' for the boolean vector $\bool,$ equal to $\sum_{k=1}^m f_k(x).$\\
        $\circuit: \{0, 1\}^m \rightarrow \{0, 1\}^{m'}$ & a boolean circuit\\
        $\circuit_l: \{0, 1\}^m \rightarrow \{0, 1\}^{m'}$ & layer $l$ of the boolean circuit $\circuit$, consisting of $m'$ boolean gates of fan-in at most two. 
    \end{tabular}
\renewcommand{\arraystretch}{1}
    \caption{Definitions of terms used in this work.}
    \label{tab:definitions}
\end{table}

We also use the following conventions for clarity:
\begin{table}[h]
    \centering
    \begin{tabular}{l l}
    $i, j \in \{1, ..., d\}$ & indices for neurons\\
    $k$, $\ell$, $p \in \{1, ..., m\}$ &  indices for features \\
    $\interf$ & amount of interference between near-orthogonal vectors\\
    $\epsilon$ & error in the read-off of a boolean feature\\
    $\sparsity$ & A bound on the ``sparsity''; we require $\boolNorm{\bool(x)}\le s\,\,\forall\,\, x\in X.$
    \end{tabular}
    \caption{Conventions used in this work.}
    \label{tab:conventions}
\end{table}

We assume our terms satisfy the following asymptotic relationships in terms of the principal complexity parameter $m$ (the number of features):
\begin{table}[h]
    \centering
    \begin{tabular}{l l}
    $d$ is polynomial in $m$&  so $d = \tOmega(m^{\alpha_+}), d = \tO(m^{\alpha_-})$ for some finite exponents $0<\alpha\le \alpha_- <\infty.$ \\
    $s$ is at worst polynomial in $m,$ & so $s = O(m^\beta).$ Note that this is different from the body, \\& where we assumed $s$ is a constant (so $\beta = 0$).\\
    $s = \tO(d^{1/3}).$ & This is a technical ``sparsity'' condition that will be useful for us.
    \end{tabular}
    \caption{Asymptotic relationships between variables in this work.}
    \label{tab:asymptotic-relationships}
\end{table}
\newpage

\section{Potential takeaways for practical mechanistic interpretability}\label{appendix:pontify}

Our motivation for studying these mathematical models is to glean insights about the computation implemented by real networks, that could have ramifications for the field of mechanistic interpretability, particularly the subfield focussed on taking features out of superposition in language models using sparse dictionary learning \citep{cunningham2023sparse,  bricken2023monosemanticity, tamkin2023codebook, bloom2024gpt2residualsaes, braun2024identifying, templeton2024scaling}. In order to render the models mathematically tractable, we have had to make idealising assumptions about the computation implemented by the networks.

\begin{enumerate}
    \item Early work on superposition \cite{elhage2022toy} suggested that it may be possible to store exponentially many features in superposition in an activation space. On the other hand, early sparse dictionary learning efforts \cite{cunningham2023sparse, bricken2023monosemanticity, bloom2024gpt2residualsaes} learn dictionaries which are smaller than even the square of the dimension of the activation space. Our work suggests that the number of features that can be stored in superposition \textit{and computed with} is likely to be around $\tO(d^2)$ (this is also the information-theoretic limit). We think that by using a dictionary size that scales quadratically in the size of the activations, while computationally challenging, this will likely lead to better performance on downstream tasks. We are heartened by more recent work by \citet{templeton2024scaling} which works with dictionaries that are closer to this size, and would encourage more systems-oriented work to scale to ever larger dictionaries.
    \item The current mainstream sparse autoencoder (SAE) architecture used by \citet{cunningham2023sparse, bricken2023monosemanticity, bloom2024gpt2residualsaes, templeton2024scaling} and others uses ReLUs to read off feature values, in accordance with the toy model of superposition of \citet{elhage2022toy} and features being ReLU-linearly represented. Our work suggests that networks may be more expressive when storing features $\epsilon$-linearly. If so, this suggests that future work should consider sparse dictionary learning with alternative activation functions that only allow for removing errors of size $\epsilon$, such as a \textit{noise-filtering nonlinearity} 
    \[\mathrm{NF}_\epsilon(x) = \begin{cases}x & |x| > \epsilon\\ 0 & |x| \leq \epsilon\end{cases}.\] 
    or nonlinearities that filter all but the k largest positive and largest negative preactivations. Notably, recent work by \citet{rajamanoharan2024improving, ProLUNonlinearity} finds suggestive evidence that the ProLU activation:
    \begin{align*}
        \mathrm{ProLU}_\epsilon(x) = \begin{cases}
            x & x > \epsilon\\
            0 & x \leq \epsilon
        \end{cases}
    \end{align*}
    outperforms the standard ReLU activation SAEs, which accords with the predictions in this work. 
    
    \item Previous work by \citet{gurnee2023finding} found some features that were represented on a \emph{small} set of neurons, even when they weren't represented on any singular particular neuron. In our constructions, feature representations end up distributed over a larger range of neurons. We expect that networks which employ superposition heavily to maximise their expressiveness are unlikely to have many sparse features that are localised to one or even a few neurons. 

\end{enumerate}

\section{Additional discussion of various feature definitions}\label{app:more-feature-discussion}

\subsection{Formal statements and proofs for facts referenced in main body}
We present formal statements and proofs that we referred to in Section~\ref{sec:feature-def}. Note that without loss of generality, we can include the activation function $a$ into our input set $X$, so we omit the use of $a$ in this section.

\begin{theorem}[Composition of linearly separable features]\label{thm:mlp-composition-linearly-separable}
    There exist a set of inputs $X$ and two features $f_1, f_2$ weakly linearly represented in $X$ such that there exists no MLP layer $\mlp$ such that either $f_1 \land f_2$ or $f_1 \lor f_2$ are linearly separable in $\mlp(x)$.
\end{theorem}
\begin{proof} (sketch)
    Let $X = [-1,1]^2$ be the unit square in $\rr^2$, and let $f_1(x) = \mathbf{1}(x_1 > 0)$ and $f_2(x)= \mathbf{1}(x_2 > 0)$ be the indicator functions of whether the first and second coordinates are greater than zero. There exists no MLP layer $\mlp: X \rightarrow \rr^d$ of any width $d$ such that $f_1 \land f_2$ is linearly separable in $\mlp(X)$.

    To show this, it suffices to notice that any MLP layer has finite Lipschitz coefficient, and that any function weakly linearly representing $f_1 \land f_2$ or $f_1 \lor f_2$ will need to have arbitrarily high Lipschitz coefficient (since there exist points that are arbitrarily close to the separating hyperplanes of $f_1$ and $f_2$. 
\end{proof}

\begin{theorem}[Composition of $\varepsilon$-linearly represented features]\label{thm:mlp-composition-eps-linear}
For any set $X$ and features $f_1, f_2$ that are $\epsilon$-linearly represented in X, there exists a two neuron MLP $\mlp: X \rightarrow \rr^2$ such that $f_1 \land f_2$ and $f_1 \lor f_2$ are $\varepsilon'$-linearly represented in $\mlp(X)$ for some $\varepsilon'$.
\end{theorem}

\begin{proof}(sketch)
    We use an MLP with two neurons $\mlp_1$, $\mlp_2$ with input weights equal to the read-off vectors of $\readoff_1, \readoff_2$. To read off $f_1 \land f_2$, we use the read-off vector $\readoff_{1\land2}$ defined by $\readoff_{1\land2}(x) = \mlp_1(x) + \mlp_1(x) - 3/4$. Similarly, to read off $f_1 \lor f_2$, we use the read-off vector $\readoff_{1\lor2}(x) = \mlp_1(x) + \mlp_1(x) - 1/4$. 
\end{proof}

In fact, by allowing for wider MLPs, it is fairly easy to construct an MLP $\mlp: X \rightarrow \rr^d$ such that $f_1 \land f_2$ and $f_1 \lor f_2$ are also $\epsilon$-linearly represented in $\mlp(X)$ (that is, with equal error). We leave the construction of this MLP as an exercise for the reader.

\section{Precise statements and proofs of theorems}\label{app:proofs}

Let $m$ be a parameter associated to the length of a boolean input. For the remainder of this section, we will work with real parameters $\alpha, \beta_{\mrin}, \beta_{\mrout}, \gamma$ which do not scale with $m$ and corresponding to scaling exponents. We impose the following asymptotic relationships on parameters $m$ (length of boolean input), $d = d_\mrin$ (width of emulating neural net), $s$ (sparsity, i.e., number of $1$ values, of suitable boolean variables), $\epsilon_\mrin$ (incoming interference, if applicable) and $\epsilon_\mrout$ (outgoing interference):
\begin{align}
m = \tOmega(r^\alpha)\\
\epsilon_\mrin = \tOmega(r^{-\beta\mrin})\\
\epsilon_\mrout = \tO(r^{-\beta\mrout})\\
s = \tO(r^\gamma).
\end{align}
More precisely, we assume that a large parameter $m$ is given and the $O(\mathrm{polylog}(m))$ scaling factors implicit in the $\tO, \tOmega$ asymptotics can be chosen in a suitable way to make the results hold. 

\subsection{Emulation of AND layer}
In this section we prove a generalization of Theorem \ref{thm:ts-and}. 

Let $\Gamma\subset \{1,\dots, m\}^{[2]}$ be the edges of a graph (here the superscript $[2]$ denotes the ``exterior power'' of a set, i.e., the set of $\binom{m}{2}$ unordered pairs). Assume that the number of edges $|E_\Gamma| = \tO(m).$ Let $\C_\Gamma:\{0,1\}^m\to \{0,1\}^{E_\Gamma}$ be the circuit with value 
$$\C_{\Gamma}(\bool)_{(k,\ell)} = \bool_k\land\bool$$ at the unordered pair $(k,\ell)\in E_\Gamma$ corresponding to an edge of $\Gamma.$
We think of $\C_\Gamma$ as the (not quite universal) circuit that takes AND's of pairs of features in $\Gamma$ and returns a boolean vector of roughly the same size.

We will show that this circuit can be emulated with suitably small interference on the output.

The proof is very similar to the proof of the error correction theorem above (Theorem \ref{thm:err-correction-appendix}), in particular with the main argument controlled by a subset $\Sigma\subset \{1,\dots, m\}\times \{1,\dots, d\},$ with $m$ the number of \emph{edges} of $\Gamma$ (i.e., outputs of the circuit). 

There are however two main differences. 
\begin{enumerate}
    \item What we read from each subset $\Sigma_{k,\ell}$ associated to an edge  $(k,\ell)\in \Gamma$ is a the result of a nonlinearity applied to a \emph{sum} of two random $\pm 1$ vectors $\phi_k, \phi_\ell$ (associated to the two inputs $k, \ell$), that returns (up to small error) the sum of neurons in of $\Sigma_{ij}$ where the signs of $\phi_k$ and $\phi_\ell$ are both $1$.
    \item To control interference issues, we need to carefully partition the graph $\Gamma$ into pieces with a certain asymptotic ``balanced'' property (see Theorem~\ref{thm:uand-superposition}).
    \item The output interference is $\tO(\sqrt{\frac{s^2}{d}}$ instead of $\tO(\sqrt{\frac{s}{d}}$ since there are $O(s^2)$ active output features (corresponding to pairs of features that are on).
\end{enumerate}

\begin{theorem}[Targeted superpositional AND]\label{thm:uand-appendix}

Let $m$ be an integer and $\Gamma\subset \{1,\dots, m\}\times \{1,\dots, m\}$ a graph. Assume we have a readoff matrix $\Readoffs_\mrin\in \mat{m}{d}$ that maps a $d$-dimensional space to an $m$-dimensional space, and let $s = o(\sqrt{m})$ be a sparsity parameter (either polynomial or polylogarithmic in $m$). Let $\epsilon_\mrin$ be an interference parameter.

Assume that we have $\epsilon_\mrin^2 m d \sqrt{d/s} = \tO(1)$ is bounded by some sufficiently small inverse polylogarithmic expression in $m.$ Then there exists a single-layer mixed emulation $\model(x) = \relu(\Win x + \Wbias)$ of the universal AND circuit $\circuit_{\mathrm{uand}}$ (together with an ``output readoff'' matrix $\Readoffs_{\mathrm{out}}$) such that $\model$ is an emulation of $\circuit_\Gamma$ on the input class $\B = \B_s$ of boolean vectors of sparsity $\le s,$ with precision $\epsilon_\mrin\to \epsilon_\mrout,$ for $\epsilon_\mrout = \tO\left(\sqrt{\frac{s^2}{d}}\right).$
\end{theorem}
Before proving the theorem, we note that our UAND statements are corollaries:
\begin{corollary}[U-AND with basis-aligned inputs]\label{thm:uand-no-superposition-appendix}
Fix a sparsity parameter $s\in \nn.$ 
Then for large input length $m$, there exists a single-layer neural network $\model(x) = \mlp(x) = \relu(\Win x + \Wbias)$ that $\varepsilon$-linearly represents the universal AND circuit $\circuit_{\mathrm{UAND}}$ on $\sparsity$-sparse inputs, with width $d = \tO_m(1/\epsilon^2)$ (i.e. polylogarithmic in $m$).  
\end{corollary}
This follows from the fact that the incoming interference $\epsilon_{\mrin} = 0$ since the incoming feature basis is basis-aligned.

\begin{corollary}[U-AND with inputs in superposition]\label{thm:uand-superposition}
Let $s \in \nn$ be a fixed sparsity limit and $\varepsilon <1$ a fixed interference parameter.
There exists a feature encoding $\Features$ and single-layer neural net $\model(x) = \mlp(x) = \relu(\Win x + \Wbias)$ with input size $m_\mrin$ and width $d = \tO(\sqrt{m_\mrin}/\varepsilon^2)$, such that $\model\circ \Features$ $\epsilon$-linearly represents $\circuit_{\mathrm{UAND}}$ on all $s$-sparse inputs $\bool$.    
\end{corollary}
This follows by restricting all but $m_\mrin = \sqrt{m}$ input features to $0$ and taking $\Gamma$ to be the complete graph on vertices $\{0,\dots, m_\mrin\}.$

Now we prove the theorem.
\begin{proof}
We begin by considering a simpler case. We say that a graph $\Gamma$ with $m$ edges is \emph{self-balanced} if each vertex has degree at most $\tO(1)$ (some fixed polylogarithmic-in-$m$ bound). 

Suppose $\Gamma$ is self-balanced. Define $A: = \sqrt{d/s}.$ For each edge $(k,\ell)\in \Gamma,$ choose at random a subset $\Sigma_{k\ell}\subset \{1,\dots, d\}$ of size within a polylog error of $A.$
Write also $$\Sigma_k = \bigcup_{\ell\mid (k,\ell)\in \Gamma} \Sigma_{k,\ell}.$$ Write down feature vectors $\feature_{k\ell} = \sum_{i\in \Sigma_{k}} \pm \vec{e}_i,$ with signs $\sigma_{k, i}$ chosen independently and randomly for each $k, i.$ For a pair $k,\ell \in \Gamma,$ define the vector $\read_{k,\ell}$ to be the indicator of the set of neurons $$\Sigma_{k,\ell}^\mrout: = \{i\in \{1,\dots, d\}\mid \sigma_{k,i} = \sigma_{\ell, i} = 1\},$$ 
Note that $|\Sigma_{k,\ell}^\mrout|$ has, o.\ n.\ p., within a polylog difference from $\frac{1}{4} |\Sigma_{k,\ell}| = \frac{A}{4}$ elements.

Write $$\feature_k^{\mrin} : = \sum_{\ell\mid (k,\ell)\in \Gamma} \feature_{k,\ell},$$
Note that this is a indicator function of a union polylog-many independently chosen sets of size $A.$ Write $\Features^{\mrin}$ for the $m\times d$ matrix with columns $\feature_k^\mrin.$

Now we define the emulation net to be 
$$\model_\Gamma(x) = \frac{4}{A}\relu(\Features^\mrin(x) - 1).$$ 

We note that (outside interference and collision errors of frequency bounded o.\ n.\ p.\ by $\tO(\epsilon_{\mrout})$,) we have $$\relu(\Features^T(\bool)) - 1)_i = \begin{cases}
1,& \exists k,\ell\in S\text{ with } i\in \Sigma_{k,\ell} \text{ and } \sigma_{k,i} = \sigma_{\ell, i} = 1\\
0,& \text{ otherwise,}
\end{cases}.$$
Here as before we take $S\subset \{1,\dots, m\}$ for the set of features that are on. 

Analogously to our proof of Lemma \ref{lem: error-correction-interfs}'s part \ref{it:err-corr-1} we see that the difference $\Features^T(\bool)-\Features^T(\Readoffs_\mrin(x))$ is (o.\ n.\ p.) bounded by $o(1),$ and thus we are done just as in the previous lemma. 

For general graphs $\Gamma,$ we might have an issue if some vertices have very high degree; if one were to try to run the same proof, their corresponding features would then admit unmanageably high interference. 

To fix this, we note that in order to emulate $\circuit_\Gamma$ it is sufficient (up to polylogarithmically increasing the number of neurons) to emulate $\circuit_{\Gamma_1},\dots, \circuit_{\Gamma_T}$ for some polylogarithmic collection of graphs $\Gamma_t$ with $\cup_t \Gamma_t = \Gamma.$ We now split an arbitrary graph $\Gamma$ into subgraphs with a nice ``balanced'' property.

Let $a,b\in \rr$ be parameters. We say that a graph is $a,b$-balanced if it is bipartite on a pair of disjoint subsets of vertices $V_0, V_1\subset \{0,\dots, m\},$ such that $|V_0| = a, |V_1| = b$ and each vertex in $V_0$ has degree at most $m/a$ and each vertex in $V_1$ has degree at most $m/b.$ We say a graph $\Gamma\subset \{0,\dots, m\}^{[2]}$ is balanced if it is $a,b$-balanced for some $a,b.$

It can be shown using an inductive argument that any graph $\Gamma$ with $m$ edges can be written as a union of $\mathrm{polylog}(m.)$ 

Now it remains to show that the theorem holds for a balanced graph. Indeed, suppose that $\Gamma$ has vertices supported on $V_0\sqcup V_1\subset \{1,\dots, m\}$ and is $a,b$-balanced. Suppose (WLOG) that $a\le b.$ Then we randomly partition the neurons $\{1,\dots, d\}$ into $a$ roughly equal sets $\Sigma_k$ for $k\in V_0$ (equivalently, we choose a random map $\{1,\dots, d\}\to V_0$ and define $\Sigma_k$ to be the preimage of $k$). We then choose for $\ell\in V_1$ the set $\Sigma_{k,\ell}$ to be a random subset of size about $\sqrt{d/s^2}$ inside $\Sigma_k,$ and define $\Sigma_\ell = \cup_{k\mid (k,\ell)\in \Gamma}.$ We finish the argument by bounding the errors in the same way as in the self-balanced case, concluding the proof.
\end{proof}

\subsection{Universal AND with inputs in superposition}

We use the conventions from Section \ref{sec:appendix-definitions}. We make an additional assumption, that our inputs $\act^{(0)}(x)$ for $x\in X$ approximately lie on a sphere of suitable radius. Note that if $m = d$ and the feature basis $\feature_i$ is an orthonormal basis, then $|\Features(\bool)| = \sqrt{\boolNorm{\bool}},$ so the $\ell_2$ norm of the embedding is the square root of the sparsity. If the sparsity $\boolNorm{\bool}$ is \emph{exactly} $s$ and the feature interference parameter $\interf$ is sufficiently small compared to the sparsity bound $s,$ we still have $|\Features(\bool)| \approx \sqrt{s}$ (with some suitable bound --- in general, it will be $\tO(\interf s^{1.5})$). If instead, we assume only that the boolean features $f_i(\bool)$ are $\epsilon$-linearly represented for suitable $\epsilon>\frac{1}{\sqrt{d}},$ in general we cannot guarantee that $|\act^{(0)}(x)| \approx \sqrt{s};$ rather, we will have $|\act^{(0)}(x)| = \tilde{\Omega}(\sqrt{s})$ since especially for small $s,$ the norm might be significantly increased by adding a large vector that is almost-orthogonal to all features (and thus doesn't affect the linear representability of the $f_i$). This observation allows us, in principle, to write down a vector with some suitable norm in $\tilde{\theta}(\sqrt{s})$ which $\epsilon$-linearly represents a very sparse boolean vector $\bool$ with $\boolNorm{\bool}<<s.$ We show how to modify inputs with unknown bounded sparsity $\boolNorm{\bool(x)}<s$ to have an (approximately) constant norm in the following section. For now, we assume in addition to $\boolNorm{\bool(x)}<s$ that all our inputs have norm equal to some $s_0 = \tO(\sqrt{s})$ up to a small error. 

\newcommand{\rad}{r}
\begin{theorem}\label{thm:uand_random_appendix-constant-norm}
Let $m, d, X, \Features, \epsilon = \epsilon_0, \interf, s$ be as in Appendix \ref{sec:appendix-definitions}. Let $r$ be a parameter so that $\rad^2 = \tO(s)$. Assume in addition to the conditions on $X,\Features$ in Appendix \ref{sec:appendix-definitions} that for any input $x\in X,$ we have $$|\act^{(0)}(x)| =  r + \tO(\frac{\sqrt{s}}{\sqrt{d}}),$$ i.e., the inputs lie approximately on a sphere of radius $r$.

Let $W\in \mat{d}{d}$ be a random weight matrix with i.i.d.\ Gaussian-distributed entries, and let $\act^{(1)}(x) = \model(\act^{(0)}(x)) : = \relu(W x)$ be the associated neural net. Then there exist some $$\epsilon^{(1)} = \tO\big(\max(s\mu, \sqrt{s}\epsilon, \sqrt{s/d})\big)$$ and $$\interf^{(1)} = \tO\big(\max(\sqrt{1/d}, \interf)),$$
such that the boolean function $f_{k\land \ell}(x) : = f_k(x)\land f_\ell(x)$ is $\epsilon^{(1)}$-linearly represented by a feature vector $\feature_{k\land\ell}^{(1)}\in \rr^d,$ outside negligible probability (in the entries of $W$). Moreover, up to rescaling by a fixed scalar, the feature vectors $\feature_{k\land\ell}$ form an almost-orthogonal collection with feature interference parameter $\interf^{(1)}.$
\end{theorem}
\begin{corollary}\label{corollary:uand_random_appendix}
The result of Theorem \ref{thm:uand_random_appendix-constant-norm} is true with the assumption $|\act^{(0)}(x)|^2 =  \rad^2 + \tO(\epsilon s)$ (that inputs are close to a sphere) replaced by $|\act^{(0)}(x)|^2 =  \tO(s),$ at the cost of increasing the depth of the neural network $\model$ from $1$ to $3$.
\end{corollary}
\begin{proof} (Of corollary.)
    This follows by chaining the neural network constructed in this theorem with the ``norm-balancer network'' constructed in Appendix \ref{app:norm-balancer} (independent from this one).
\end{proof}
The idea of the proof of Theorem \ref{thm:uand_random_appendix-constant-norm} is derived from the quadratic activations case, $\model(\vec{x}) = Q(W\vec{x}),$ where $Q$ is the function that squares entries of a vector coordinatewise. Let $a_k^i = W(\feature_k)^i$ (for $i\in \{0,\dots, d-1\}$) be the coordinates of the preactivation vector $W(\feature_k)$ associated to the $k$th boolean bit. 

One can show using the theory of quadratic forms that the readoff vector $R_{k,\ell}^i = a_k^i a_\ell^i$ gives a valid readoff direction to show $\epsilon$-strong linear separation of the boolean expression $\bool_k\land \bool_\ell$ (o.\ n.\ p.). We will show that a similar strategy works for an arbitrary (reasonable, and in particular nonlinear) activation function, including ReLU.

Write down the unnormalized model $\model^u(\vec{x}) : = \relu(W(\vec{x})).$ Define $\feature_k' = W \feature_k$ to be the preactivation under this model of $\feature_k.$
Define the unnormalized readoff matrix for the UAND coordinate associated to the pair of features $k,\ell$ as follows: $$\readoff_{k,\ell}^i = \sign((\feature_k')_i \cdot (\feature_\ell')_i),$$ where $\text{sign}(x)$ is the sign function that returns $-1, 0, 1$ depending on whether $x$ is negative, $0$ or positive, respectively. 
\begin{remark}
Note that as we care about the existence of a linear representation rather than a learnable formula for it, the readoff doesn't have to depend continuously on the parameters. However having continuous dependence is also possible; in particular, it would also be reasonable to make the dependence continuous; indeed, the readoff vector with coordinates $a_k^i \cdot a_\ell^i$ (same as for quadratic activations) would also work, with an alternative normalization; the important property of the readoff function is that it is odd in each of the $x$ and $y$ coordinates independently, and that it does not have wild asymptotic behavior. We use the discrete ``sign'' function for the readoff for convenience.
\end{remark}

The crucial observation is the following simple lemma. 
For a given input $x,$ let $\act(x)$ be the corresponding embedding. Let $$\act(x)^\Lambda : = \act(x)-f_k(x)\feature_k-f_\ell(x)]\feature_\ell$$ (the ``hat'' notation denotes that we are ``skipping'' information about features $k$ and $\ell$ in the embedded input $\act(x);$ it linearly represents the modification of the boolean vector $\bool(x)$ that zeroes out the $k$th and $\ell$th coordinates). 
\begin{lemma}
Suppose $\Features, k, \ell,$ and $\bool$ are fixed. Then in the context of the theorem above, the unnormalized readoff $\Readoffs^u_{k,\ell}(\model(\Phi(\bool)))$ is a sum of $d$ i.i.d.\ variables of the form $F(x_i,y_i,z_i),$ where $F(x,y,z) = \sign(x) \sign(y) \relu(\bool_k(x) x+\bool_\ell(x) y+z)$ and the triple $(x_i,y_i,z_i)$ is drawn from the distribution $\N(0,\Sigma)$ where 
\[\Sigma = \begin{pmatrix} 
    ||\feature_k||_2^2 & \feature_k\cdot \feature_\ell & \feature_k\cdot \vec{x}^\Lambda\\
    \feature_k\cdot \feature_\ell & ||\feature_\ell||_2^2 & \feature_\ell\cdot \vec{x}^\Lambda\\
    \feature_k\cdot \vec{x}^\Lambda&\feature_\ell\cdot \vec{x}^\Lambda&||\vec{x}^\Lambda||^2
    \end{pmatrix}.\]
\end{lemma}
\begin{proof} 
Write $x_i = (\feature_k')_i, y_i = (\feature_\ell')_i, z_i = W \act(x)^\Lambda$ be the neuronal coordinates of the corresponding activations. Then $(R_{k,\ell}^u)_i = \sign(x_i)\sign(y_i)$ and 
$$\model^u(\act(x))_i = \relu\big(W(\act(x))_i\big) = \relu(\bool(x)_k x_i+\bool(y)_k y_i+z_i).$$ It remains to show that $(x_i, y_i, z_i)$ are drawn according to the Gaussian distribution $\N(0,\Sigma).$ This follows from the standard result that applying a Gaussian-distributed matrix with entries in $\N(0,1/d)$ to a collection of vectors $\vec{v}_1,\dots, \vec{v}_n$ is distributed as a (possibly singular) Gaussian with PSD covariance matrix $\Sigma_{k\ell} = \vec{v}_k\cdot \vec{v}_\ell.$
\end{proof}

Now our interference bounds imply that the triple $(x_i, y_i, z_i + \bool_k x_i + \bool_\ell y_i)$ are distributed according to a matrix of the form 
\[\begin{pmatrix}
1+ O(\mu)&O(\mu)&\bool_k+O(\epsilon)\\
O(\mu)&1+O(\mu) &\bool_\ell + O(\epsilon)\\
\bool_k+O(\epsilon)&\bool_\ell + O(\epsilon) & \rad^2 + \tO(s/\sqrt{d}).
\end{pmatrix}\]
Let $s': = \rad^2-\bool_k-\bool_\ell$ and $\rad' : = \sqrt{s'}.$

Now o.n.p., we can assume that $x_i,y_i \in \tO(1)$ and $z_i\in \tO(\rad).$ Since $F$ grows linearly, we see that $F(x_i, y_i, z_i) \in \tO(\rad)$ o.n.p. We can now apply Bernstein's inequality \ref{theorem:bernstein} to get that, o.n.p.,
$$\sum_{i=1}^d F(x_i, y_i, z_i) = d[\ee_{(x, y, z)\sim \N(0,\Sigma)}f(x,y,z)+\tO(\rad/\sqrt{d})].$$ Now since $\rad= \tO(\sqrt{s})$ and $|(\rad')^2-\rad^2|$ is an integer equal to at most $2$ (the sum of two feature readoffs of $\act$), the error term in the Bernstein inequality is bounded by $\tO(\rad'/\sqrt{d}).$ It remains to estimate the expectation 
$$E: = \ee_{(x, y, z)\sim \N(0,\Sigma)}F(x,y,\bar{z}).$$

Assume that $\bool(x)$ has nonzero coordinates other than at $k,\ell,$ so that $\rad' = \Omega(1)$ (the case where $\bool(x)$ only has nonzero coordinates on a subset of $\{k,\ell\}$ can be handled similarly and more easily). In this case, we add a new notation $$F'(x,y,z') : = F(x,y,s'z') = \sign(x)\sign(y) \relu(r'\bar{z}+\bool_k x + \bool_\ell y),$$ where the third input of $F$ is rescaled to make the distribution on $(x,y,z')$ closer to the identity Gaussian. Let $\Sigma'$ be the distribution on $(x,y,z'),$ given by  $$\Sigma' = \text{diag}(1,1,(r')^{-1})\Sigma\text{diag}(1,1,(r')^{-1}).$$ Since the two differ by a reparametrization, the expectation of $F'$ on $\N(0,\Sigma')$ is equal to the expectation of $F$ on $\N(0,\Sigma).$

Let $X' = \N(0,\Sigma')$ and $X_0' = \N(0,\idx),$ both on $\rr^3.$ Our various interference bounds imply that the difference $\Sigma-\idx$ is bounded by 
$$\delta: = \tO\big(\max(\frac{\sqrt{s}}{\sqrt{d}}, \frac{\epsilon}{\sqrt{s}}, \interf)).$$ 
This means that the total variational difference between $X$ and $X'$ is bounded by $O(\delta).$ Now the expectation $F'$ on $X,X_0$ are not affected, up to negligible terms, by $(x,y,z)$ outside some constant $\tO(1),$ and here $F'$ is bounded by $\tO(r).$ Thus we have $$|\ee_{(x,y,z')\sim X}F'(x,y,z')-\ee_{(x,y,z')\sim X_0}F'(x,y,z')| = \tO(r\delta).$$ 

It remains to estimate the mean 
$$E_0: = \ee_{(x,y,z')\sim X'_0}F'(x,y,z') = \ee_{(x,y,z)\sim X_0} F(x,y,z),$$
where $X_0 = \N(0, \text{diag}(1,1,(d')^2)).$

    Up to symmetry, we have three cases depending on the $k$ and $\ell$ coordinates of $\bool = \bool(x)$ associated to our input:
    \begin{itemize}
        \item $\bool_k = \bool_\ell = 0,$
        \item $\bool_k = 0, \bool_\ell = 1,$
        \item $\bool_k = \bool_\ell = 1.$
    \end{itemize}
    The expectation calculation in the first two cases are trivial: if $\bool_k$, is zero, then each $F$ is odd in the $x$, resp., $y$ coordinate, so since the distribution $X_0$ is independent Gaussian, the mean is $$E_0 = 0.$$ 
    
    It remains to consider the case $\bool_k = \bool_\ell = 1,$ i.e., the ``interesting'' case where $\land(\bool_k,\bool_\ell) = 1.$ We write down the integral expression
    \begin{align}E_0:=\mathbb{E}_{(x,y,z)\sim X_0}Q_i(x,y,z)= \int \sign(x)\sign(y) \relu(x+y+z) p_0(x,y,z) dxdydz,
    \end{align}
    for $p_0(x,y,z)$ the pdf of $X_0=\N(0,\text{diag}(1,1,s')).$
    We would like to show this value is positive and bound it from below (to show eventually that the mean in the CLT dominates the errors). We use $x,y$-symmetry to rewrite the integral as 
    $$A = 2\int_{x\le y}\sign(x)\sign(y)\relu(x+y+z)p_0(x,y,z).$$ Since the independent Gaussian $p_0(x,y,z)$ is symmetric in the $x$ and $y$ coordinates, we can collect $\pm x, \pm y$ terms together to write $$E_0 = 2\int_{0\le x\le y}p(x,y,z)\big(\relu(x+y+z)-\relu(x-y+z) -\relu(-x+y+z)+\relu(x+y+z)\big).$$ We split the domain up further into five terms, $$E_0 = A^{--}+A^{-}+A^0 + A^+ A^{++},$$ into regions on which the relus are constantly $0$ or nonnegative linear functions:
    \begin{align*} &A^{--} &=& 2\int_{0\le x\le y, z\le -x-y} p_0(x,y,z)dxdydz\cdot 0\\
    &A^- &=& 2\int_{0\le x\le y, -x-y\le z\le x-y} p_0(x,y,z) (x+y+z)\\
    &A^0 &=& 2\int_{0\le x\le y, x-y\le z\le y-x} p_0(x,y,z)dxdydz\,\big((x+y+z)-(-x+y+z)\big) \\ 
    &&=& 2\int_{\ldots} p_0(x,y,z) dxdydz\, (2x)\\
    &A^+ &=& 2\int_{0\le x\le y, y-x\le z\le x+y} p_0(x,y,z)dxdydz\,(x+y+z)-(-x+y+z)-(x-y+z) \\ 
    &&=& 2\int_{\ldots} p_0(x,y,z) dxdydz\,\big(x+y-z\big)\\
    &A^{++} &=& 2\int_{0\le x\le y, z\ge x+y} p_0(x,y,z) dxdydz\,\big((x+y+z)-(-x+y+z)-(x-y+z)+(-x-y+z)\big) \\ 
    &&=& 0.
    \end{align*}
    Note in particular that each term above is nonnegative on its domain (for $A^+,$ this is because the domain includes the inequality $z\le x+y$). Thus in particular, $E\ge A^0.$ Since the integrand is positive, we can get a lower bound by restricting the domain:
    $$A^0\ge 2\int_{x\le 1, y\ge 2, -1\le z\le 1} 2 p_0(x,y,z)dxdydz,$$
    using that the integrand is $2x\ge 2.$ This is, equivalently, twice the probability that $|x|\ge 1,|y|\ge 2, |z|\le 1,$ for $(x,y,z)$ drawn from $p_0(x,y,z) = \sigma_{0,1}(x)\sigma_{0,1}(y)\sigma_{0,\rad^2-2}(z).$ By independence of $p_0,$ this is a product of $3$ terms. The probability distributions on $x,y$ are fixed unit Gaussians, so the corresponding terms are $O(1),$ and so the mean has (up to an $O(1)$ constant) the same asymptotic as the third term, which is $$P_{z\sim \sigma_{0,\rad^2-2}}(|z|<1) = O(1/\rad) = \tTheta(1/\sqrt{s}).$$ 
    The Bernstein bound applied to $d$ i.i.d.\ such variables now gives us o.n.p.
    $$\sum_{i=1}^d F(x_i, y_i, z_i)_{(x_i,y_i, z_i)\sim X_0} = d\cdot E_0 + \sqrt{d}\tO(\rad).$$
    Incorporating error terms, we get 
    $$\readoff_{k,\ell}^u(\model^u(\act(x))) = d\cdot E_0 + \sqrt{d} \tO(\rad) + d\tO(\rad\delta).$$
    We now normalize:
    \begin{align}
        \model(\act) : = \frac{\model^u(\act)}{\sqrt{d} E_0}\\
        \readoff_{k,\ell} : = \frac{\readoff_{k,\ell}^u}{\sqrt{d}}.
    \end{align}
Then if $f_k(x)\land f_\ell(x) = 1,$ then (o.n.p.)
$$\readoff_{k,\ell}(x) = 1 + \frac{\tO(\rad)}{\sqrt{d}} + \tO\rad\delta.$$
Alternatively if $f_k(x)\land f_\ell(x) = 0,$ the expectation is zero and we are left with the error term, 
$$\readoff_{k,\ell}(x) = \frac{\tO(\rad)}{\sqrt{d}} + \tO\rad\delta$$
The theorem follows. \qed

\subsection{Norm-balancer network}\label{app:norm-balancer}
In this section, we prove a technical result that was needed in the previous section. Namely, at one point we assumed that the norm of our inputs $\act_0(x)$ are (o.n.p., and up to a multiplicative error of $1+\tO(\frac{1}{\sqrt{d}})$) equal to a specific value $\lambda,$ which is related to the sparsity by a bound of the form $\lambda = \tO(\sqrt{s}).$ It is not difficult to guarantee this if we know the exact sparsity of the sparse boolean vector $s_{exact} = \boolNorm{\bool_0}.$ However, in the process of chaining together multiple boolean circuits, we would like to allow the exact sparsity of intermediate layers to vary (so long as it is bounded by $s$), even if the exact sparsity of the input layer is fixed. In this section we give a two-layer neural network mechanism that allows us to circumvent this issue by modifying all inputs $\act_0(x)$ to have roughly the same norm, equal to some specific value $\sqrt{s_0} = \tO\sqrt{s}.$ 

We note that while it seems plausible that real neural networks share properties in common with the past two artificial neural nets we constructed (error correction and universal AND), the neural net constructed here

\newcommand{\balance}{\mathrm{balance}}
\begin{theorem}\label{thm:norm-balancer}
    Let $s_0 = \tO(\sqrt{d})$ be a sparsity parameter. 
    There exists a 2-layer neural net $\balance_{s_0}:\rr^d\to \rr^d$ depending on random parameters, with hidden layers of width $O(d),$ with the following property.
    
    Suppose that $\feature_1,\dots, \feature_d$ is a collection of features of length $<2,$ and $\act_x$ is an input satisfying $|\act_x|<\sqrt{s_0}.$ Then 
    \begin{enumerate}
        \item $|\balance(\act_x)| = \sqrt{s_0} \cdot (1+ \tO(1/\sqrt{d}))$
        \item $\act_x\cdot \feature_k - \balance(\act_x)\cdot \feature_k = \tO(\frac{\sqrt{s_0}} {\sqrt{d}}).$
    \end{enumerate}
\end{theorem}
\begin{proof}
    Let $W\in \mat{d}{d}$ be a random square matrix, with entries drawn independently from $\sigma(0,1/d^2).$ Define the function $N(\act) = \sum_{i=1}^d\relu(W x)_i.$ Then $N(\act)$ is a sum of $d$ i.i.d. random variables of the form $N_i = \relu(x)\mid x\sim \sigma(0,|\act|/d).$ Applying arguments similar to those used in the proof of the previous theorem, we see that $N_i$ has norm $c\cdot |\act|/d,$ for $c>0$ the absolute constant $$c = \mathbb{E}_{x\sim \sigma(0,1)} \relu(x) = \frac{1}{2 \sqrt{\pi}}.$$ The variance of $N_i$ is $O(|\act|^2)/d,$ and $N_i$ is bounded o.n.p. by $\tO(|\act|).$ Thus Bernstein's inequality implies that, o.n.p., $$N(\act) = \sum_{i=1}^d N_i = c\cdot |\act| + \tO(|\act|/\sqrt{d}).$$ 
    Now $|\act| < s_0 = \tO(\sqrt{s}),$ so $N(\act) = |\act|+\tO(\epsilon).$ Let $f(y) = \sqrt{s_0-y^2}$ (for $|y|\le \sqrt{s_0}$), a semicircle of radius $\sqrt{s_0}$ viewed as a function of a real variable. Define the piecewise-linear function $f_{PL}$ given by splitting the semicircle into $d$ equal arcs, and connecting the endpoints of the arcs (extending the first and last arc linearly outside the domain of definition). The difference between the values of $f$ on the endpoints of each arc is bounded by its arclength, which is $O(\sqrt{s_0}/d).$ Thus $|f(x)-f_{PL}(x)| = O(\sqrt{s_0}/d)$ (in fact, much better asymptotic bounds are possible.) Now $f_{PL}$ is a sum of $d$ ReLUs, thus it is a scalar-valued function which can be expressed by a width-$d$ neural net. Now choose a random ``approximately unit'' vector $v\in \rr^d$ according to the Gaussian $v\sim \sigma(0,1/\sqrt{d}).$ Now we define the neural net $\balance(\act): = \act+f_{PL}(N(\act))v.$ Since both $N(\act)$ and $f_{PL}$ can be expressed as width-$d$ neural nets, $\balance$ can be expressed as a width-$O(d)$ neural net. Now since $v$ is a random vector, we have, o.n.p., $$v\cdot \act = \tO(|\act|/\sqrt{d})$$ and $v\cdot \feature_k = \tO(\frac{1}{\sqrt{d}}).$ Since there are at most polynomially-many (in $r$) features, the ``negligible probablity'' exceptions remain negligible when combined over all features. The bound $N(\act) = \tO(\sqrt{s})$ thus implies both bounds in the theorem.
\end{proof}

Let $\alpha(x), \beta(x,y)$ be functions. Let $W$ be random and $\Phi$ be a matrix of features. Fix $k, \ell\in \{0,\dots, m-1\}.$ Let $\bool\in \{0,1\}^m$ be a boolean vector. Let $\bool_{kl} = \bool_k \feature_k + \bool_\ell \feature_\ell,$ and $\bool' = \bool-\bool_{k\ell}.$ Outside negligible probability, we know that $\Phi(\bool_{k\ell})\cdot \Phi(\bool') = \tO(\epsilon).$ This means that if $\epsilon = \tTheta(1/\sqrt{d})$ and we apply a random matrix $W$ then we still have $W\Phi(\bool')\cdot W\Phi(\bool_{k\ell}) = \tO(\epsilon)$ (outside negligible probability). Define $\vec{x}_{kl} = W\Phi(\bool_{kl})$ and $\vec{x}' = W\Phi(\bool').$ Since random matrices are $O(d)$-invariant, we can assume WLOG that these are drawn independently and randomly from appropriate Gaussian distributions \emph{EXPAND}. Specifically, $\vec{x}_{kl}$ is drawn from a distribution with variance $2$ and $\vec{x}'$ is drawn from a distribution with variance $O(s).$ 

Define 
$$\model(\vec{x}) = \alpha(W(\vec{x})),$$
and define 
$$R^i_{k\ell} : = \beta(\feature^i_k, \feature^i_\ell).$$
\begin{lemma}
    For suitable choices of a piecewise-linear function $\alpha$ and some function $\beta$ (both depending on $s$) we can guarantee that $R_{k,\ell}\cdot \model(\Phi(\bool)) = \bool_k\land \bool_\ell+\tO(\epsilon_{\mrout}).$
\end{lemma}
\begin{proof}
As explained above, we can assume that $x^i_{k\ell}, (x^i)'$ are drawn from independent boolean distributions with variance respectively $\frac{2}{d}, \frac{s}{d}.$ Define $X = \sigma(0, \frac{s}{d}I)$ to be the Gaussian variable with variance $\frac{s}{d}.$
    Define 
    $$\Delta_i(x): = \alpha\left(x+\vec{x}^i_{k\ell})-\alpha(x)\right).$$
    Write
    $$\model_\Delta(\vec{y})^i : = \Delta_i(\vec{y}).$$
    Then $\model(\vec{x}) = \model(\vec{x}')+\Delta_i(\vec{x}').$
    It remains to prove the following sublemma:
    \begin{lemma}
    (Outside negligible probability:)
        \begin{align}R_{k\ell}\cdot \model(\vec{x}') = \tO(\epsilon_\mrout)\\
        R_{k\ell}\cdot \Delta_i(\vec{x}') = \bool_k\land \bool_\ell + \tO(\epsilon_\mrout)
        \end{align}
    \end{lemma}
    We start with the first expression. We have 
    \begin{itemize}
        \item $\vec{x}_i'$ random from Gaussian $X$, variance $s/d.$
        \item $\alpha(\vec{x}_i')$ random, bounded by $B$ (o.n.p. bound for $\alpha$ on $X$).
        \item From POV of $x':$ we know $(x,y)$ random Gaussian, variance $1/d.$
        \item So $R_{kl}\cdot \model(\vec{x}')$ is the sum of $d$ samples of $\beta(x,y)\alpha(z)$ for $x,y,z$ from appropriate Gaussians.
        \item WTS: $\pm$ symmetric in independent way, variance $\tO(\epsilon_{\mrout})/d,$ bounded (onp) by $\tO$ of stdev (check if this bound correct for Azuma inequality). For this (modelling on quadratic case): choose $\beta$ to be $\pm$ symmetric in either coordinate independently, and appropriately bounded.
    \end{itemize}
    For the second expression, we treat two cases, namely $(\bool_k,\bool_\ell) \in \{(1,1), (0,1)\}.$ We do not need to treat other cases as $(1,0)$ follows by symmetry and $(0,0)$ is trivial.
    Start with $(1,1)$ case, so $\bool_k\land \bool_\ell = 1.$ We then have
    \begin{itemize}
        \item Want $$E\left(\beta(x,y)\Delta_{x,y}(z))\right)$$ to be $1$.
        \item Above bounded to make Azuma ok (prob enough to check $\Delta = O(1)$ and use Azuma bounds from previous).
    \end{itemize}
    Final case, $(0,1).$
    \begin{itemize}
    \item Want $E(\left(\beta(x,y)\Delta_x(z))\right)$ to be $1$.
    \item This follows from $\pm$ symmetry of $\beta$ (and bounds as above).
    \end{itemize}
\end{proof}

\subsection{Error correction layers}\label{sec:error-correction}

\begin{theorem}\label{thm:err-correction-appendix}
Suppose we are in the context of Appendix \ref{sec:appendix-definitions}. Then 
there exists a polylog constant $K = K(d)$ and a single-layer neural net $\model(x) = v_1 + W_1 \relu (v_0 + W_0(x))$ and a feature matrix $\Features^{(1)}\in \mat{d}{m}$ such that if $\epsilon (= \epsilon^{(0)}) < K \frac{d^{1/4}}{m^{1/2} s^{1/4}},$ then for each input $x,$ o.n.p., the feature $\feature_k^{(1)}$ linearly separates the boolean function $f_k$ on the activation $\act^{(1)}(x) = \model(x),$ with error $$\epsilon^{(1)} = O\big(\log(d)\cdot \frac{\sqrt{s}}{\sqrt{d}.}\big)$$
Moreover, we can choose the new feature vectors $\phi_k^{(1)}$ such that they have feature interference bounded by 
$$\interf^{(1)} = \tO\big(\frac{\sqrt{s}}{\sqrt{d}}\big).$$
\end{theorem}

\begin{proof}
\newcommand{\errvec}{\overrightarrow{\mathrm{err}}}
\newcommand{\round}{\mathrm{round}}

We begin by defining an unnormalized version of the output feature matrix. Define $p = \frac{1}{\sqrt{ds}},$ a probability parameter.
Let $\Features^{(1),u}\in \mat{m}{d}$ be a matrix of entries $M^i_k$ drawn uniformly from the ternary random variable $\begin{cases}
p(M^i_k = 1)&=    p/2 \\
p(M^i_k = -1)&=    p/2 \\
p(M^i_k = 0)&=    1-p 
\end{cases}.$

Let $\Gamma\subset \{0,\dots, m\}\times \{0,\dots, d\}$ be the set of nonzero values of $\Features^{(1), u}.$ Note that (o.n.p.), it has size 
$$|\Gamma| = m\sqrt{\frac{d}{s}}+\tO(1).$$

We think of this as a graph, connecting each feature $k$ to a set of (approximately $\sqrt{\frac{d}{s}}$) neurons it ``activates'', $\Gamma_k\subset \{1,\dots,d\}.$ We also write $\Gamma_i\subset \{1,\ldots, m\}$ for the set of features connected to the $i$th neuron. 

Let 
$$\round_{[0,1]}(x) : = 3\left(\relu(x-1/3)-\relu(x-2/3)\right),$$
the piecewise-linear function that maps $\rr$ to the interval $[0,1]$ and is non-constant only on the interval $(1/3,2/3).$

Now for any integer, define 
$$\round_{[0,a]}(x) : = \round_{[0,1]}(x) + \round_{[0,1]}(x-1) + \dots + \round_{[0,1]}(x-a+1),$$ and similarly, 
$$\round_{[-a,a]}(x) : = \round_{[0,a]}(x) - \round_{[0,1]}(-x).$$ This is a piecewise-linear ``staircase'' function with the following properties:
\begin{itemize}
    \item $\round_{[-a,a]}(x)\in [-a,a]$ for all $x\in \rr$ and
    \item $\round_{[-a,a]}(n+\epsilon) = n,$ whenever $n\in [-a,a]$ is an integer and $\epsilon<1/3.$ 
\end{itemize}
Thus for all sufficiently small values $x,$ the function $\round$ will ``round'' $x$ to the nearest integer, so long as the nearest integer is less than $1/3$ away; hence its name. By construction, the function $\round_{[-a,a]}(x)$ is a sum of a $4a$ ReLUs.

We will use for our nonlinearity the function 
$$\text{round}(x) = \round_{[-2,2]}(x):$$ 

\begin{center}
\begin{tikzpicture}
\begin{axis}[
    axis equal,
    axis lines = middle,
    xlabel = $x$,
    ylabel = {$\text{round}(x)$},
    domain=-2:2,
    samples=200,
    legend pos=outer north east
]
\addplot [
    thick,
    blue,
    domain=-2:2,
] {3*max(0, 5/3 + x) - 3*max(0, 4/3 + x)+3*max(0, 2/3 + x) - 3*max(0, 1/3 + x) + 3*max(0, x - 1/3) - 3*max(0, x - 2/3)+3*max(0, x - 4/3) - 3*max(0, x - 5/3)-2};
\end{axis}
\end{tikzpicture}
\end{center}
(Using larger intervals $[-a,a]$ in our nonlinearity $\round_{[-a.a]}$ would give slightly stronger results, but won't be needed.)




Now we define the unnormalized neural net model as follows: 
\begin{align}\model^u(x) : = \mathrm{round}(\Features^{(1), u}\left(\Features^{(0)}\right)^T(x)).
\end{align}
Finally, we normalize: \begin{align}\model(x) : = \frac{\model^u(x)}{\sqrt{d/s}} \\ \Features^{(1)} : = \frac{\Features^{(1), u}}{\sqrt{d/s}}.
\end{align}

For each feature $k\in \{1,\dots, m\}$ in an input $x$, the unnormalized neural net $\model^{(1),u}$ roughly does the following.
\begin{enumerate}
    \item ``Reads'' the feature $\phi_k$
    \item ``Writes'' $1$s in all neurons $i\in \Gamma_k$ connected to $k$ assuming $\phi_k$ is present
    \item ``Rounds'' each neuron which is close to $-2, -1, 0, 1$ or $2$ to the closest integer.
\end{enumerate}
At the end, we hope to obtain a vector with exactly the entry $M_k^i\in \pm 1$ for each $k$ with $f_k(x) = 1$ and zero elsewhere. If we're lucky and there are no issues with excess interference and no pairs of active features $k, \ell$ that share a neuron $i\in \Gamma_k\cap \Gamma_\ell,$ the result of this computation will be $\Features^{(1),u}(\bool(x)),$ and its error can then be controlled by understanding the interference of the new normalized feature matrix $\Features^{(1)}$. 

In order to make this work, we need to control two types of issues: 
\begin{itemize}
    \item \emph{Collision}: it's possible that two simultaneously active features $k,\ell$ with $f_k(x) = f_\ell(x) = 1$ share some neurons, so some of the entries of $\feature^{(1),u}_k + \feature^{(1),u}_\ell$ have ``colliding'' information from the $k$th and $\ell$th neurons that gives the wrong answer after getting rounded to one of $\{-2,-1, 0, 1,2\}.$
    \item \emph{Interference}: it's possible that, even if $\Gamma^k$ are disjoint for all features $k$ appearing in $\bool(x),$ the various interference terms shift the value far enough from the ``correct'' value in $\{-1,0,1\}$ that the ``round'' function does not successfully return it to its original position.
\end{itemize}

These are controlled by the two parts of the following lemma.
\begin{lemma}
    \begin{enumerate}
        \item\label{it:neuron-error} For any $x\in X$, we have o.n.p.:
        $$||\left(\Features^{(1,u)}\left(\Features^{(0)}\right)^T(x) - \Features^{(1),u}\bool_x\right)||_\infty = o(1).$$
        \item\label{it:feature-error} For any boolean $\bool$ with sparsity $\boolNorm{\bool}<s$, we have (o.n.p.) the difference 
        $$\errvec_{\text{collision}} : = \round(\Features^{(1,u)}(\bool))-\Features^{(1,u)}(\bool)\in \rr^d$$ has all unnormalized feature readoffs $$\feature_k^u\cdot \errvec_{\text{collision}} = \tO\max(1,\sqrt{s^3/d}).$$
    \end{enumerate} 
\end{lemma}
\begin{proof}
    Note that the two results are both about $\ell_\infty$ errors, but in two different spaces, namely in the space $\rr^d$ with the neuron basis for part (\ref{it:neuron-error}) and in the space $\rr^m$ with the feature basis for part (\ref{it:feature-error}). We start with part (\ref{it:neuron-error}). Since there is a polynomial number of neurons, bounding the $\ell_\infty$ error o.n.p. is equivalent to bounding the difference for each coordinate: $$E_i(x):=\left(\Features^{(1,u)}\left(\Features^{(0)}\right)^T(x) - \Features^{(1),u}\bool(x)\right)\cdot \vec{\boldsymbol{e}}_i.$$
    This difference is a linear combination of the errors $\feature_k^{(0)}\cdot x,$ with coefficients given by the matrix coefficients $\left(\Features^{(1,u)}\right)_i^k,$ with $i$ fixed and $k$ varying. For a pair $(i,k)\in \Gamma,$ let $\sigma(i,k)\in \pm 1$ be the sign of the corresponding matrix coefficient (which is chosen independently at random in the random variable-valued definition of our neural net). We then have
    $$E_i(x) = \sum_{k\in \Gamma^i} \sigma(i,k) x\cdot \feature_k^{(0)}.$$ By assumption, $x\cdot \phi_k\le \epsilon^{(0)}.$ Since the signs are chosen independently at random, we can bound this value o.n.p.\ by the Bernstein inequality, Theorem \ref{theorem:bernstein}, with discrete variables $X_k = \sigma_{i,k} x\cdot \feature_k^{(0)}.$ Here $k$ is indexed by a $|\Gamma^i|$-element set. By definition of $\Features^{(1)},$ each element $\{1,\dots, m\}$ has probability $p = \frac{1}{\sqrt{sd}}$ of being in $\Gamma^i,$ so $$|\Gamma^i| = \frac{m}{\sqrt{sd}}+\tO\left(\frac{\sqrt{m}}{\sqrt{sd}}\right) = \tO\left(\frac{m}{\sqrt{sd}}\right).$$ Since all these random variables are bounded by $\epsilon^{(0)}$ in absolute value, Bernstein's inequality implies that o.n.p., $$E_i = O\left(\epsilon^{(0)}\cdot \left(\frac{m}{\sqrt{sd}}\right)^{1/2}\right),$$ giving part (\ref{it:neuron-error}) of the lemma. 

    To prove the second part, note that the ``ground truth'' activation ${\act}_{\text{ground}}: = \Features^{(1,u)}\bool$ is an integer-valued vector with coefficients $({\act}_{\text{ground}})_i = \sum_{k \in \Gamma^i\cap \bool} \sigma_k.$ It is changed by applying the $\round$ function if and only if this sum is $>2$ in absolute value, i.e., if it is a ``collision'' (i.e., contained in the intersection) of at least $3$ subset of the form $\Sigma_k.$ The expectation of the number of such overlaps a given neuron $i\in \{1,\dots, d\}$ can be can be bounded by 
    $$O(\frac{s^3}{(\sqrt{sd})^3}) = O\left(\frac{s^{3/2}}{d^{3/2}}\right).$$
    Thus the coefficients of the error vector $$(\errvec_{\text{collision}})_i = ({\act}_{\text{ground}})_i-\round({\act}_{\text{ground}})_i$$ are drawn i.i.d. from a distribution with mean $0$ (as it is symmetric) and variance bounded by $\tO(\frac{s^{3/2}}{d^{3/2}}),$ which is absolutely bounded by $\tO(1).$ In other words, we have o.n.p.\ that this vector has at most $$\tO\max\left(1,\left(s^{3/2}\sqrt{d}\right)\right)$$ entries all bounded by $\tO(1),$ and with independently random signs. When we take the dot product with another unnormalized feature vector we are left with an error bounded by 
    $$\epsilon_{\text{collision}}\cdot \feature_k^{(1),u} = \tO\max\left(1,\left(\frac{s^{3/2}}{d^{1/2}}\right)\right),$$ completing the proof.
\end{proof}
Now we can finish the proof. The interference bound in the lemma implies that o.n.p., the $d$-dimensional vector
$$\Features^{(1), u}(\bool) - \Features^{(1), u}(\Features^{(0,T)}(x))$$ has all coefficients bounded by $o(1),$ an in particular, bounded by $1/3.$ Since the LHS has all integer entries, this means that 
$$\round(\Features^{(1), u}(\bool)) = \round(\Features^{(1), u}(\Features^{(0,T)}(x)))$$ (As the ``round'' function is constant on $[n-1/3, n+1/3]$ for any integer $n$). 

Since we have assumed that $s<d^{1/3}$ (in \ref{sec:appendix-definitions}), the asymptotic term $s^{3/2}/d^{1/2}$ in the collision error bound is bounded by $1$, so o.n.p., $\epsilon_{\text{collision}}\cdot \feature_k^{(1),u} = \tO(1).$ Finally, when we normalize, both sides of the dot product get multiplied by $A = s^{1/4}/d^{1/4},$ and so after normalizing the coresponding bound gets multiplied by $\sqrt{s}/\sqrt{d},$ and we get the expression (o.n.p.):
$$\feature_k^{(1)}\cdot \left(\model(x)-\Features^{(1)}(\bool(x))\right) = \tO(\sqrt{s}/\sqrt{d}).$$ 

Finally, by a similar argument to the collision proof, we see that the unnormalized dot product $\Features^{(1),u}(\bool(x))\cdot \feature_k$ is $\sqrt{d}/\sqrt{s}$ up to an error of $\tO(1),$ so the error m

We claim that the pair $(\model, \Features^{(1)})$ satisfies (o.n.p.) the conditions for the error-correction circuit above, for some appropriate relationships between the values $d,\epsilon^{(0)}, \epsilon^{(1)}$ depending on $m,$ satisfying asymptotic inequalities of the form 
$$\epsilon^{(0)} = \tO\big(\frac{d^{1/4}}{m^{1/2}s^{1/4}}\big),$$
$$\epsilon^{(1)} = \tO\left(\frac{\sqrt{s}}{\sqrt{d}}\right),$$ 
$$\interf^{(1)} = \tO\left(\frac{\sqrt{s}}{\sqrt{d}}\right).$$

\begin{lemma}\label{lem: error-correction-interfs}
For a suitable choice of $\epsilon_\mrin$ as above we can guarantee that:
\begin{enumerate}
\item\label{it:err-corr-1} If $\errvec\in \rr^m$ has $\errNorm{\errvec}<\epsilon_\mrin,$ then $\errNorm{\Features(\errvec)} = o(1),$ o.\ n.\ p. (Note that the latter value is an $\ell^\infty$ norm in the neuron basis.) 
\item\label{it:err-corr-2} If $\bool$ is boolean and $s$-sparse, then $\frac{\Features}{A}(\round(\Features(\bool)) \approx_{\epsilon \mrout} \bool,$ o.\ n.\ p. 
\end{enumerate}
\end{lemma}

To get part (\ref{it:err-corr-1}) above, observe that for any neuron index $i,$ we have 
$$\Features(\errvec)_i = \sum_{k\mid k\in \Gamma^i} \sigma_{k,i}\errvec_k,$$ where we define $\Gamma^i:= \{k\mid (k,i)\in \Gamma\}.$
Since the signs $\sigma_{k,i}$ are random and independent, this is a sum with random signs of numbers of absolute value $<\epsilon_{\mrin}.$ From the Azuma inequality, we see that (o.\ n.\ p.) $\Features(\errvec)_i = \tO(\errvec \cdot \sqrt{|\Gamma_i|}).$ Since the $\Gamma$ was chosen randomly, o.\ n.\ p. 
$$|\Gamma_i| = \tTheta(|\Gamma|/d) = \tTheta(d^{\frac{1-\gamma}{2}}) = o(\epsilon_{\mrin}^{-2}).$$ 
The last statement follows from comparing exponents in the two sides, and the freedom of choice of polylog term in $\epsilon_\mrin.$

For part (\ref{it:err-corr-2}) above, observe that $\left(\Readoffs_\mrout(\round(\Features(\bool)))\right)_k$ is the average over the set $\Gamma_k = \{i\mid (k,i)\in \Gamma\}$ of $$a_i:=\round\left(\sum_{\ell\in S} \Features_{\ell,i}\right)$$ where $S$ is the set of features that are on in $\bool,$ of size $|S|\le s.$ We want to compare this to $\bool_k$, which is $1$ if $k\in S$ and $0$ otherwise. We expect (for $i\in \Gamma_k$) that $a_i = 0$ if $\bool_k = 0$ and $a_i = 1$ if $\bool_k = 1.$ Since $\round()$ always returns a value of absolute value $\le 1,$ we can bound the error by twice the number of incorrect values. We get errors of two types.
\begin{enumerate}
\item \emph{Interference error}, from neurons that are on when they should be off. I.e., when $a_i \neq 0$ despite $\bool_k = 0.$ 

\item \emph{Collision error}, from neurons which should be on but are $0$ (or have wrong sign) due to contributions from both $S_k$ and another feature. 
\end{enumerate}
Either of these errors happens when $\Gamma_k$ and $\bigcup_{\ell\in S'}\Gamma_\ell$ intersect for $S' = S\setminus\{k\},$ the set of nonzero values of $\bool$ not equal to $k$. Now $\Gamma_k$ has $\tO(d^{\frac{1-\gamma}{2}})$ nonzero entries and $\bigcup_{\ell\in S'}\Gamma_\ell$ has at most $\tO(d^{\gamma+\frac{1-\gamma}{2}})$ entries; since each subset $\Gamma_k$ is independently random, we see (o.\ n.\ p.) that the intersection has at most $\tO(\frac{d^{\frac{1+\gamma}{2}}d^{\gamma+\frac{1+\gamma}{2}}}{d}) = \tO(1)$ entries, and the average is indeed $\tO(\epsilon_\mrout).$ 

This completes the proof of the lemma. The theorem follows. Indeed, suppose that $x\in \rr^{d \mrin}$ is a vector with $\Readoffs_\mrin(x)\approx_{\epsilon \mrin} \bool$ for $\bool\in \{0,1\}^m$ an $s$-sparse boolean vector. Setting $\errvec = \Readoffs_\mrin(x)-\bool,$ part  (\ref{it:err-corr-1}) implies that $$\Features\circ \Readoffs_{\mrin}(x)-\Features(\bool)$$ has coefficients at most $o(1);$ since $\Features(\bool)$ has integer entries, this means that applying $\round$ to both sides produces the same results. Part (\ref{it:err-corr-1}) then implies that the RHS $\Features(\bool)$ has sufficiently small interference.
\end{proof}
\begin{corollary}[Lemma \ref{lem:error-correction-layer}]
    For sufficiently small input interfefrence there exists a 1-layer MLP that returns (outside negligible probability) an encoding of the same boolean vector with low interference ($1/\sqrt{d}$ assuming low sparsity parameter).
\end{corollary}
\begin{proof}
    This follows from the theorem in the case $\gamma = 0,$ i.e., when the sparsity parameter $s$ is polylog in $m.$
\end{proof}

\section{Theoretical Framework and Statistical Tools} \label{app:more-formalism}

Here we provide statistical definitions and lemmas required for our proofs in Appendix~\ref{app:proofs}.

\subsection{Negligible probabilities}
Most results in this paper are proven \emph{outside negligible probability}. This is a standard notion in complexity theory and cryptography \cite{bellare2002note}, with the following formal definition:

\begin{definition}\label{def:negligible-probability}Let $\{E_n\}_{n=1}^\infty$ be a sequence of events parameterized by $n$. 
We say that $E_n$ is true \textbf{with negligible probability} (w.\ n.\ p.) if for any polynomial exponent $c \in \nn$, there exists some constant $N_c \in \nn$ such that $P(E_n) < O(n^{-c})$ for all $n > N_c$. Similarly, we say that $E_n$ is true \textbf{outside negligible probability} (o.\ n.\ p.) if its complement $\overline{E_n}$ is true with negligible probability.

If $E_n = E_n(x)$ depends on an input in some set $X$, when we say $E_n(\bool)$ is true with negligible probability for all fixed inputs $x\in X$ we implicitly assume that there is an explicit constant $C_n < O(n^{-c})$ as above that bounds the probability of $E_n(x)$ for each valid input $x\in X$.
\end{definition}
Intuitively, the reason why this probability is ``negligible'' is that the union of polynomially many events of negligible probability also has negligible probability. As we never consider more networks requiring more than polynomially many operations, we can ignore events of negligible probability at each step when performing asymptotic analysis, which greatly simplifies our proofs. 

\begin{example}
    Let $\bool$ be a random boolean vectors of length $n$. Then outside negligible probability, $\bool$ has between $n/2 + \log(n)\sqrt{n}$ and $n/2-\log(n)\sqrt{n}$ zeroes. 
\end{example}
This follows from the central limit theorem. (Note that if we used $\sqrt{\log(n)}\sqrt{n},$ the result would be false!)

For cases where the event is a bound on a random function (as above), we can combine ``negligible probability'' notation and big-$O$, as well as big-$\tO$ notation, as follows.
\begin{definition}
    Suppose a function $f(x) = f_n(x)$ depends on the complexity parameter $n$ and a fixed input $x\in X$ and is valued in random variables\footnote{The input can be an ``empty input'', i.e., $f$ is itself a random variable depending only on $n$}. Let $g(x)\ge 0$ be a deterministic function\footnote{or a constant depending on $n$ if $x$ is an empty input}. Then we say that $$f(x) = \tO(g(x))$$ if there exists a polylog constant $K_n = O(\text{polylog}(n))$ such that, for any input $x,$ the event $|f(x)|<g(x)K(x)$ is true outside negligible probability.
\end{definition}
This lets us rephrase the previous example as ``for $\bool$ a random boolean vector of length $m,$ we have $\sum \bool_k(x) = m/2 + \tO(\sqrt{m}).$'' We also list the following result, which will be important for us.
\begin{lemma}\label{lemma:gaussian-orthogonality}
    Let $d\in \nn$ be a complexity parameter. Let $v\in \N(0,\idx/d)$ be a Gaussian-distributed random vector in $\rr^d,$ and let $x\in \rr^d$ be a fixed input vector. Then, outside negligible probability, we have 
    \begin{enumerate}
        \item $|v| = 1+ \tO\left(\frac{1}{\sqrt{d}}\right)$
        \item $v\cdot x = \tO\left(\frac{|x|}{\sqrt{d}}\right).$
    \end{enumerate}
\end{lemma}
\begin{proof}
    The first statement is standard (and follows from the central limit theorem applied to the real variable $\N(0,1)^2$). The second statement follows from the fact that sums of Gaussian random variables are Gaussian (and variance adds). 
\end{proof}
Note that this in particular implies a version of the Johnson-Lindenstrauss lemma: 
\begin{corollary}
        Suppose $m$ is a polynomial function of $d$ (which we take to be the complexity parameter), and suppose $\feature_1,\ldots, \feature_m\in \rr^m$ are random vectors drawn from $\N(0,\idx/d).$ Then outside negligible probability, $$\feature_k\cdot \feature_\ell = \begin{cases}
            1+\tO(1/\sqrt{d}),& k=\ell\\
            \tO(1/\sqrt{d}),  & k\neq \ell.
        \end{cases}$$
\end{corollary}
\begin{proof}
    We are checking polynomially many (namely, $O(m^2)$ with $m$ polynomial in $d$) statements, thus by the union bound, it suffices to show that each is true outside negligible probability. The corollary now follows by inductively on $k$ applying \ref{lemma:gaussian-orthogonality} to $\feature_k\cdot \feature_k = |\feature_k|^2$ and $\feature_k\cdot \feature_\ell\mid \ell<k,$ in the latter case taking the vector $\feature_\ell$ as fixed.
\end{proof}
Before continuing, we record the following simple result, which will allow us to convert ``negligible probability'' results to our existence results in the body.
\begin{theorem}
Suppose that $s = O(1)$ is a constant sparsity parameter, $\model$ is a model in a fixed class that depends on some random parameters, and a property $P(\model, x)$ holds outside negligible probability for all inputs $x = \Features(\bool)$ corresponding to boolean inputs $\bool\in \{0,1\}^m$ of sparsity $s.$ Then there exists a model $\model$ such that the property $P(\model, \bool)$ holds for all boolean inputs $\bool.$
\end{theorem}
\begin{proof}
This follows from the union bound, since the number of possibly inputs $\bool$ with sparsity $s$ is $\binom{m}{s}<m^s$ (and negligible probability goes to zero faster than any inverse polynomial).
\end{proof}
\begin{remark}
    For every ``negligible probability'' statement we encounter, it is straightforward to check that, up to decreasing the asymptotic parameters in appropriate $\tO$-asymptotic assumptions in the variables involved, we can guarantee for a stronger statement to hold: namely, for any fixed $c,$ we can guarantee that the negligible probability $p$ asymptotically satisfies $p = O(\exp(-\log(m)^c)).$ Thus (by another union bound), statements that are true with negligible probability for any boolean input $\bool$ of size $\boolNorm{\bool} = \tilde{O}(1)$ (at most polylogarithmic in $m$) can be made to hold for all such parameters $\bool,$ for an appropriate choice of parameters.
\end{remark}
\subsection{Concentration inequalities}
Concentration inequalities (in the sense we use here) bound tail probabilities of sums of random variables which are either i.i.d.\ or ``close to'' i.i.d. in some sense. As we only care about $\tO$-type precision in our error bounds (i.e., up to polylog factors) and we need statements to be true only outside negligible probability, we are able to get away with very weak versions of bounds which exist in general with much more precision; both of the results we need follow from the Bernstein inequality for martingales (which subsumes the Azuma inequality).

\begin{theorem}[Coarse Bernstein bound]\label{theorem:bernstein}
    Suppose that $X_1,\dots, X_n$ are a real random variable bounded by a constant $M$, which are either i.i.d. or form the difference sequence of a Martingale, i.e., $\mathbb{E}(X_i\mid X_1,\dots, X_{i-1}) = 0.$ Then 
    $$\sum x_i = n\interf + \tO(M\sqrt{n})$$ outside negligible probability, uniformly in the $X_i$. In other words, there exists a polylogarithmic sequence of constants $K_n = O(\text{polylog}(n))$ such that the probability
    $$P\big(|\sum_{i=1}^n (x_i-[X_i])|<K_n\cdot M\sqrt{n}\big)\le P_n$$
    for some sequence $P_n$ that goes to zero faster than any polynomial function in $n$.
\end{theorem}
\begin{proof}
    This follows from Bernstein's theorem, \cite{bernstein1924modification}. In fact, both statements also follow from the simpler Azuma-Hoeffding inequality.
\end{proof}
\begin{corollary}\label{cor:gauss-bernstein}
Let $V = \rr^a$ be a vector space, with $a = O(1)$ a constant (we will use $a= 1$ and $a = 3$). Let $\Sigma\in \mat{a}{a}$ be a fixed symmetric positive-definite matrix, with $X = \N(0,\Sigma)$ the corresponding distribution. Let $f:V\to \rr$ be a fixed function with subpolynomial growth in $x$, and let $\interf = [f(x), x\sim X]$ be the mean of $f$ on $x$ drawn from this distribution. Let $x_1,\dots, x_m$ be a collection of variables drawn from i.i.d.\ copies of $\N(0,\Sigma).$ Then o.n.p., $\sum f(x_i) = m\interf + \tO(\sqrt{m}),$ where the polylogarithmic constant in $\tO$ depends on $f.$
\end{corollary}
\begin{proof}
    Since $f$ has polynomial growth, $f(x) < K(1+|x|^c)$ for some constants $C, d.$ Thus o.n.p. in $m$, $f(x)\le (c\log(Km))$ (note that $f(x)$ doesn't depend on $m;$ we're just saying that $P\big(f(x)\le d\log(m)\big)$ goes to $0$ faster than any polynomial function in $m;$ in fact this probability is $O(m^{-\log(m)})$). Let $M = c\log(Km).$ Then the concentration theorem above implies that $$\sum_{i=1}^m (f(x_i)-[f(x_i)]) = \tO(M) = \tO(1),$$ since $M = \tO(1).$
\end{proof}

\subsection{Precise and mixed emulations}\label{app:more-formalism:emulations}
The parameters in the models $\model$ in the proofs of our emulation results depend on random matrices of $\pm 1$'s and $0$'s, hence can be understood as suitable random variables. In terms of this point of view, we make the following definition.

Suppose that $\circuit:\{0,1\}^m\to \{0,1\}^{m'}$ is a boolean circuit with input size $m.$ We always assume that the output size $m'$ and the depth are at most polynomial in $m.$ Let $\B\subset \{0,1\}^m$ be a class of inputs (usually characterized by a suitable sparsity property). Let $\epsilon<1$ be an interference parameter. 
\begin{definition}
An \emph{$\epsilon$-precise emulation of $\circuit$} (on input class $\B$) is a triple of data $(\Features, \model, \Readoffs)$ all possibly depending on random parameters where $\Features\in\mat{d_\mathrm{in}}{m}$ is a feature matrix, $\Readoffs\in \mat{m'}{d_\mathrm{out}}$ is a readoff matrix and $$\model:\rr^{d\mathrm{in}}\to \rr^{d\mathrm{out}}$$ is a (not necessarily linear) function given by a neural net, with the following property:

\

{\noindent \sl For any $\bool\in \B,$ we have, outside negligible probability,}
$$\errNorm{\Readoffs\circ \model\circ\Features(\bool)-\circuit(\bool)}<\epsilon.$$
\end{definition}
Importantly, we do not consider the boolean circuit $\circuit$ or the input $\bool\in \B$ to be random variables, and the randomness involved in the negligible probability statement is purely in terms of the parameters that go into the emulation scheme $(\Features, \model, \Readoffs).$ In particular, this guarantees that if the boolean input $\bool$ is generated in a non-random way (e.g., adversarially), an emulation nevertheless guarantees (in the ``negligible probability sense'') safe performance on $b$ so long as the parameters of the emulation were chosen randomly.

It will be useful to extend the notion of emulation to one which correctly approximates $\circuit$ on inputs $x\in \mathbb{R}^{d\mathrm{in}}$ which represent a boolean input $\bool\in \{0,1\}^m$ not in the sense of ``pure superposition'' $x = \Features(\bool)$ but in the sense of ``read-off'', $$\errNorm{\Readoffs_{\mathrm{in}}(x) - \bool}<\epsilon_{\mathrm{in}}.$$ Here $\Readoffs_{\mathrm{in}}\in \mat_{m}{d_{\mathrm{in}}}$ is a readoff matrix that should be thought of as a noisy inverse to the feature matrix on sparse inputs. Formally, we make the following definition. Here we will assume that the matrix $\Readoffs_{\mathrm{in}}$ was generated at an earlier stage of the computation, and does not depend on random variables. 

Fix a circuit $\circuit:\{0,1\}^{m}\to \{0,1\}^{m'},$ a class of inputs $\B \subset \{0,1\}^m,$ and an ``input readoff'' matrix $\Readoffs_{\mathrm{in}}\in \mat{m}{d_\mathrm{in}}.$ Let $\epsilon_\mrin, \epsilon_\mrout$ be two interference parameters.
\begin{definition}
     A \emph{mixed emulation of $\circuit$ with precision $\epsilon_\mrin\to\epsilon_\mrout$} (on input class $\B$ and relative to a fixed input readoff matrix $\Readoffs_\mrin$) is a pair of data $(\model, \Readoffs_\mrout)$ both possibly depending on random parameters where $\Readoffs\in \mat{m'}{d_\mathrm{out}}$ is a readoff matrix and $$\model:\rr^{d\mathrm{in}}\to \rr^{d\mathrm{out}}$$ is a (not necessarily linear) function given by a neural net, with the following property:

\

{\noindent \sl For any boolean input $\bool\in \B$ and $x\in \rr^{d \mrin}$ satisfying 
$$\errNorm{\Readoffs_\mrin(x)-\bool}<\epsilon_\mrin,$$
we have, outside negligible probability,}
$$\errNorm{\Readoffs\circ \model(x)-\circuit(\bool)}<\epsilon_\mrout.$$
\end{definition}
\begin{remark}
    Note that if it is impossible to accurately represent $\bool$ via the matrix $\Readoffs,$ i.e., to satisfy $\errNorm{\Readoffs(x)-\bool}<\epsilon_\mrin,$ then the notion of mixed emulation is vacuous (any neural net would satisfy it for tautological reasons). We will generally apply this notion in contexts where such representations are possible (for example, with via a suitable feature matrix $x = \Features(\bool)$).
\end{remark}
Here as before we do not consider the boolean circuit $\circuit$ or the input $\bool\in \B$ to be random variables, and in addition the representation $x$ and the input readoff matrix $\Readoffs_\mrin$ are assumed fixed. 
So the randomness involved in the negligible probability statement is purely in terms of the parameters that go into the pair $(\model, \Readoffs).$

\end{document}